\documentclass{article}

\PassOptionsToPackage{numbers, compress}{natbib}
\usepackage{amsmath,amssymb,amsthm,fullpage,mathrsfs,pgf,tikz,caption,subcaption,mathtools,cite}
\usepackage[ruled,vlined,linesnumbered,noresetcount]{algorithm2e}
\usepackage{natbib}
\usepackage{amsmath,amssymb,amsthm,mathtools}
\usepackage[utf8]{inputenc} 
\usepackage[T1]{fontenc}    
\usepackage{hyperref}       
\usepackage{url}            
\usepackage{booktabs}       
\usepackage{amsfonts}       
\usepackage{nicefrac}       
\usepackage{microtype}      
\usepackage[margin=1in]{geometry}
\usepackage{bbm}
\usepackage{enumitem}
\let\oldnl\nl
\newcommand{\nonl}{\renewcommand{\nl}{\let\nl\oldnl}}

\newtheorem{theorem}{Theorem}[section]
\newtheorem{corollary}[theorem]{Corollary}

\newtheorem{lemma}[theorem]{Lemma}
\newtheorem{definition}[theorem]{Definition}

\newtheorem{example}[theorem]{Example}

\newtheorem{observation}[theorem]{Observation}

\newcommand{\R}{\mathbb{R}}
\newcommand{\Pbb}{\mathbb{P}}

\newcommand{\np}{\lambda}
\newcommand{\us}{ S}

\newcommand{\zo}{ \mathbf{0}}
\newcommand{\bigo}{\mathcal{O}}
\newcommand{\Hcal}{\mathcal{H}}
\newcommand{\id}[1]{\mathbbm{1}_{#1}}
\DeclareMathOperator{\ConvHull}{ConvHull}
\DeclareMathOperator{\vol}{Vol}
\DeclareMathOperator{\cone}{Cone}
\DeclareMathOperator{\sign}{sign}
\DeclareMathOperator{\poly}{poly}
\DeclareMathOperator{\polylog}{polylog}
\DeclareMathOperator{\gtnc}{GTNC}
\DeclareMathOperator{\tnc}{TNC}
\DeclareMathOperator{\bin}{Bin}
\newcommand\underrel[2]{\mathrel{\mathop{#2}\limits_{#1}}}
\newcommand{\yields}[1]{\underrel{#1}{\implies}}
\newcommand{\euler}{e}
\newcommand{\q}[1]{``{#1}''}
\newcommand{\E}{\mathbb{E}}

\title{Noise-tolerant, Reliable Active Classification with Comparison Queries}

\author{%
  Max~Hopkins\thanks{Department of Computer Science and Engineering, UCSD, California, CA 92092. Email: \texttt{nmhopkin@eng.ucsd.edu}. Supported by NSF Award DGE-1650112.},
  Daniel~Kane\thanks{Department of Computer Science and Engineering / Department of Mathematics, UCSD, California, CA 92092. Email: \texttt{dakane@eng.ucsd.edu}. Supported by NSF Award CCF-1553288 (CAREER) and a Sloan
Research Fellowship.},
    Shachar~Lovett\thanks{Department of Computer Science and Engineering, UCSD, California, CA 92092. Email: \texttt{slovett@cs.ucsd.edu}. Supported by NSF Award CCF-1909634.},
    Gaurav Mahajan\thanks{Department of Computer Science and Engineering, UCSD, California, CA 92092. Email: \texttt{gmahajan@eng.ucsd.edu}.}
}

\begin{document}

\maketitle

\begin{abstract}
With the explosion of massive, widely available unlabeled data in the past years, finding label and time efficient, robust learning algorithms has become ever more important in theory and in practice. We study the paradigm of active learning, in which algorithms with access to large pools of data may adaptively choose what samples to label in the hope of exponentially increasing efficiency. By introducing comparisons, an additional type of query comparing two points, we provide the first time and query efficient algorithms for learning non-homogeneous linear separators robust to bounded (Massart) noise. We further provide algorithms for a generalization of the popular Tsybakov low noise condition, and show how comparisons provide a strong reliability guarantee that is often impractical or impossible with only labels - returning a classifier that makes no errors with high probability.
\end{abstract}
\newpage
\section{Introduction}
Due to the now ubiquitous presence of massive unlabeled datasets, recent years have seen an explosion in the search for computationally efficient, noise tolerant learning strategies that minimize the required amount of labeled data to learn a classifier. \textit{Active learning} is a formalization of the PAC-learning paradigm for unlabeled data. In active learning, the learning algorithm has access both to either a stream or pool of unlabeled data, and an oracle which can label the data on request. The complexity of learning certain classes is then defined by their query complexity, the number of oracle calls required to almost learn the classifier with high probability. The goal in active learning is to adaptively choose data to send to the oracle in such a way that one uses much fewer queries than in the labeled case.
\\
\\
While active learning saw initial success in the noise-free regime with simple concept classes such as thresholds in one dimension, lower bounds \cite{dasgupta2005analysis} soon showed that important classes such as linear separators gave no improvement over PAC-learning, even in only two dimensions. However, subsequent work showed that slight tweaks to the model could overcome this barrier. Balcan and Long \cite{balcan2013active} showed that by assuming that the data was drawn from a log-concave distribution -- a wide set of distributions including Gaussian distributions and uniform distributions over convex sets, learning homogeneous (through the origin) linear separators could be done in exponentially fewer queries than in the PAC model. Later, Balcan and Zhang \cite{s-concave} extended this to the more general class of s-concave distributions, a generalization of log-concavity that includes fat-tailed distributions as well. Rather than restricting the power of the adversary, Kane, Lovett, Moran, and Zhang \cite{kane2017active} studied the effect on query complexity of empowering the learner. By allowing the learner to ask more complicated questions of the oracle, such as comparing two points, Kane {et al.} \cite{kane2017active} showed that non-homogeneous linear separators in two-dimensions can be learned in exponentially fewer labeled samples than the PAC case. Later, Kane, Lovett, and Moran \cite{kane2018generalized} extended this to higher dimensions using a complicated set of queries, and Hopkins, Kane, and Lovett \cite{AID} did the same by assuming weak concentration and anti-concentration on the distribution -- conditions once again satisfied by s-concave distributions. 
\\
\\
While query efficient algorithms in high dimensions are an important step towards the use of active learning on real world data, it is equally important that algorithms be computationally efficient and noise tolerant. In an early work, Castro and Nowak \cite{castro2006upper} provided query efficient algorithms for thresholding in one dimension in the presence of bounded (Massart \cite{massart2006risk}) and unbounded (Tsybakov \cite{mammen1999smooth}) noise under the uniform distribution on $[0,1]$. Soon after, Balcan, Broder, and Zhang \cite{balcan2007margin} extended these results to $d$-dimensional homogeneous hyperplanes over a uniform distribution on a ball. Years later, Hanneke \cite{hanneke2011rates} offered a more general analysis for Tsybakov noise based off of the distributional complexity measure the disagreement coefficient, and later with Yang provided a distribution-free analysis \cite{hanneke2015minimax}. In another vein of work, Balcan and Long \cite{balcan2013active} provided an algorithm for learning $d$-dimensional homogeneous hyperplanes over nearly isotropic log-concave distributions with optimal query complexity for Tsybakov noise \cite{wang2016noise}, a result which was later extended by Awasthi, Balcan, Haghtalab and Urner \cite{awasthi2015efficient} to be computationally efficient for Massart noise when the distribution is restricted to uniform over the unit ball. Similarly, Balcan and Zhang \cite{s-concave} gave a computationally efficient algorithm for learning the more difficult adversarial noise model over s-concave distributions. Concurrently, Xu {et al.} \cite{xu2017noise} proposed using comparison queries as a sub-routine in previous algorithms to deal with noise in a computationally efficient manner, improving the overall query complexity along the way.
\\
\\
The comparison based methods of Xu {et al.} \cite{xu2017noise}, however, do not carry over to the algorithmic technique proposed by Kane {et al.} \cite{kane2017active} for learning non-homogeneous linear separators. Kane et al.'s technique is based upon logical inference. Viewing concept classes as the sign of an underlying family of functions, they build a learner via a linear program with constraints given by query solutions. As a result, the learners created by Kane et al.'s method actually fall into a stronger model than PAC-learning called Reliably and Probably Useful (RPU)-learning \cite{rivest1988learning}, variants of which have been studied more recently under a variety of names (e.g. KWIK learning \cite{li2011knows}, perfect selective classification \cite{el2012active}, or confident learning \cite{kane2017active}). In this model, the learner is not allowed to err, but may instead output \q{I don't know} a small fraction of the time. While Kane et al.'s RPU-learner is computationally efficient, it is not tolerant to noise -- the linear program is sensitive to errors in both labels and comparisons. This raises a natural question: can the inference based algorithms of Kane {et al.} be extended to noisy scenarios, and if so, does a strong reliability guarantee remain?
\\
\\
In this work we answer these questions in the positive for Massart and Tsybakov noise. In both cases our algorithms satisfy a noisy version of RPU-learning: with high probability the learner makes no errors at all. Due to their similarity to RPU-learners, we call learners that satisfy this property \textit{Almost Reliable and Probably Useful} (ARPU). Indeed, taking the limit of our reliability condition returns exactly the RPU model. Our work provides the first query and computationally efficient algorithm for PAC or ARPU-learning non-homogeneous linear separators in the presence of Massart noise over s-concave distributions, as well as more generally for hypothesis classes with finite inference dimension or small average inference dimension. In addition, we provide the first algorithm for ARPU-learning non-homogeneous linear separators under the Tsybakov Low Noise Condition.
\\
\\
Similar to how Xu {et al.} \cite{xu2017noise} use comparisons as a subroutine for correcting label errors, we use an approximate sorting scheme (modified from a seminal work from Braverman and Mossel \cite{braverman2009sorting} on sorting with noisy comparisons) to create a small set of points whose labels \textit{and} comparisons are correct with high probability. We then feed this cleaned set into an inference LP, and repeat the process in a boosting style algorithm based off of the framework of \cite{kane2017active}. By carefully curating the cleaned set at each step, we are able to use a symmetry argument from \cite{kane2017active} to prove that our learners have good coverage, while the guarantees of \cite{braverman2009sorting} and the inference framework give reliability.
\\
\\
Our algorithms require the use of comparison queries, an addition which we show is necessary in many cases for active PAC and ARPU-learning. Along with recalling lower bounds from \cite{AID} which show comparisons are necessary for efficiently active learning non-homogeneous hyperplanes, we show that in the noiseless case it is impossible to ARPU-learn the uniform distribution over $S^1$ in a finite number of label queries. Further, even with the addition of a margin assumption we show the existence of simple distributions which require a number of label queries that is exponential in dimension. Because Massart noise and certain instantiations of Tsybakov noise subsume the noiseless case, these results prove the existence of a large gap between labels and comparisons for noisy ARPU-learning.
\\
\\
Our paper proceeds as follows. In Sections \ref{prelims}, \ref{intro-results}, and \ref{techniques} we cover preliminaries, our main results, and our main techniques respectively. In Section~\ref{Sec:massart} we present query and computationally efficient algorithms for ARPU-learning hypothesis classes with finite inference dimension or super exponential average inference dimension under the Massart noise model, as well as a lower bound for ARPU-learning $S^1$ using only labels. In Section~\ref{Sec:GTNC} we present algorithms for ARPU-learning linear separators with margin and finite inference dimension or over distributions with weak distributional conditions under the Tsybakov Low Noise Condition, as well as a lower bound for ARPU-learning a corresponding distribution with margin using only labels

\subsection{Preliminaries}\label{prelims}
\subsubsection{Basic definitions} A hypothesis class is a pair $(X,\Hcal)$, where $X$ is a set, and $\Hcal$ is a class of functions $h\colon X \to \R$. Each function $h\in \Hcal$ is called a \textit{hypothesis}. We refer to $C_\Hcal = \{\sign(h)\colon h\in \Hcal\}$ as the associated \textit{concept} class. For example, when $\Hcal$ is the class of $\R^d \to \R$ affine functions, then the associated concept class $C_\Hcal$ is the class of $d$-dimensional half-spaces.

We consider the binary classification problem, where we want to predict the binary label $y$ for each instance $x$. We assume access to an underlying unknown distribution $D_X$ over $X$ and a label oracle $Q_L$. Querying $Q_L$ with unlabeled $x \in X$ generates a label $Q_L(x)$, drawn from unknown distribution $\Pbb(Q_L(x)|x)$. Note that querying $Q_L$ on the same point again would generate the same answer. We use notation $D_L$ to denote the joint distribution over examples $x$ and labels from $Q_L$:
\[
    \Pbb_{D_L}(x,y) = \Pbb_{D_X}(x) \Pbb(Q_L(x)=y|x)
\]

\subsubsection{PAC-Learning}
Probably Approximately Correct (PAC) learning is a probabilistic framework due to Valiant \cite{valiant1984theory} and Vapnik and Chervonenkis \cite{vapnik1974theory} for learning adversarially chosen classifiers and input distributions. In this model, given a set $X$ and a set $\Hcal$ of hypotheses $h: X \to \R$, an adversary first chooses distribution $D_L$ over $X\times Y$ with the marginal distribution $D_X$ over $X$. If $Y=\sign(h^\star(X))$ for some $h^\star\in \Hcal$, we call this realizable case learning. With no knowledge of the choice of distribution the learner draws labeled samples from $D_L$ with the goal of outputting $c = \sign(h)$ for some hypothesis $h \in \Hcal$ which minimizes loss over $D_L$:
\begin{equation*}
    L_{D_L}(c) \triangleq \mathbb{E}_{(x,y)\sim D_L}[\id{c(x) \neq y}].
\end{equation*}
In the realizable case, a hypothesis class $(X,\Hcal)$ is called PAC-learnable if $\forall \varepsilon, \delta$, there exists a learner $A$, where no matter the choice of the adversary, outputs a concept $A(S)$ such that:
\begin{equation*}
    \Pr_{S\sim D_X^n} [L_{D_L}(A(S)) \ge \varepsilon] \leq \delta.
\end{equation*}
Here $n=n(\varepsilon,\delta)$ is called the sample complexity, and must be $\poly(\frac{1}{\varepsilon},\frac{1}{\delta})$ for $(X,\Hcal)$ to be PAC-learnable.
\subsubsection{RPU-Learning}
Reliable and Probably Useful (RPU) learning is an alternative learning framework in which the learner is not allowed to make errors, but may instead respond \q{I don't know}, notated by \q{$\bot$}. Introduced by Rivest and Sloan \cite{rivest1988learning}, RPU learning was later studied under the name of Perfect Selective Classification by El-Yaniv and Weiner \cite{el2012active}, and confident learning by Kane, Lovett, Moran, and Zhang \cite{kane2017active}. Since it is easy to make a reliable learner by simply always outputting \q{$\bot$}, our learner must be useful, and with high probability cannot output \q{$\bot$} more than a small fraction of the time. Let $A$ be a reliable learner and $A(S)$ be the concept returned by the learner $A$ on training sample $S$, then we define the loss of $A(S)$ as the measure of unlearned samples:
\begin{equation*}
    L_{D_L}(A(S)) \triangleq \mathbb{E}_{(x,y)\sim D_L}[A(S)(x) = \bot].
\end{equation*}
We will commonly refer to $1-L_{D_L}(A(S))$ as the \textit{coverage} of $A(S)$. Sample complexity and learnability are then defined analogously to PAC-learning. Note that any point which is not labeled \q{$\bot$} by an RPU-learner is labeled correctly. 

\subsubsection{Comparison Queries}
Following the framework of \cite{kane2017active}, our learner will have access to more information than just the label of a point. We focus on one particularly natural additional query, the ability to compare points. A comparison query measures the relative distance of two points to the decision boundary. In other words, say that our goal is to identify photographs of diseased vs healthy patients. A comparison query asks: \q{which patient looks healthier?}. 
Formally, given an underlying function $h \in \Hcal$ and two points $x_1,x_2$, a comparison query asks which one of $h(x_1),h(x_2)$ is bigger. Equivalently:
\[
\sign(h(x_1) - h(x_2)) \geq 0?
\]
Similar to our label oracle $Q_L$, we define a comparison oracle $Q_C$. Querying $Q_C$ with two points $x_1,x_2 \in X$ generates a comparison result $Q_C(x_1, x_2)$, which is drawn from an unknown distribution $\Pbb(Q_C(x_1, x_2)| x_1,x_2)$. Along with their added theoretical power \cite{kane2017active,AID}, comparison queries are already used in practice in recommender systems \cite{rec} and ranking systems \cite{braverman2009sorting}, and in some scenarios have better accuracy than label queries \cite{xu2017noise}.

\subsubsection{Inference Dimension}
Inference dimension is a combinatorial complexity measure introduced by Kane {et al.} \cite{kane2017active} to characterize the query complexity in active learning when the learner is allowed to ask a more complicated set of questions. Given a set of binary queries $Q$, let $Q(S)$ denote the answers to all such queries on the sample $\us$. Let $S\subseteq X$ be an unlabeled sample. For $x\in X$ and $h\in \Hcal$, let
\[
    Q(\us) \yields{h} x
\] 
denote the statement that answers to binary queries from $Q$ on the sample $\us$ determine the label of $x$, when the learned concept is $\sign(h(x))$, corresponding to an hypothesis $h$. We will often say for shorthand that $S$ ``infers'' $x$, and sometimes drop the underlying classifier $h$. In this case the underlying function is assumed to be the Bayes optimal classifier. Inference dimension with respect to some query set $Q$ is defined as follows.
\begin{definition}[Inference dimension]
\label{def:ID}
The inference dimension of $(X,\Hcal)$
is the minimal number $k$ such that
for every $\us \subseteq X$ of size $k$,
and every $h\in \Hcal$ there exists $x\in \us$ such that
\[Q(\us\setminus\{x\}) \yields{h} x.\]
If no such $k$ exists then the inference dimension of $(X,\Hcal)$ is defined as $\infty$.
\end{definition}
Inference dimension is a worst case measure. Since we will be dealing with varying levels of distribution dependence, we will also take advantage of an average case version of inference dimension introduced in \cite{AID}. 
\begin{definition}[Average Inference Dimension]
\label{def:AID}
We say $(D_X, X, \Hcal)$ has average inference dimension $g(n)$, if: 
\[
\forall h\in \Hcal, Pr_{S \sim D_X^n}[ \nexists x\in S \ \text{s.t.} \ Q(S \setminus \{x\}) \yields{h} x] \leq g(n)
\]
\end{definition}
Average inference dimension is used to prove that the inference dimension of a finite sample drawn from $D_X$ cannot be too large with high probability. This allows us to build query efficient algorithms for hypothesis class with infinite inference dimension by proving that large finite samples do not take too many queries to learn with high probability.

\subsubsection{Noisy Learning}
Before we discuss our relaxation of RPU-learning, we formalize the presence of noise in our distributions. Given a hypothesis class $(X,\Hcal)$, we assume the Bayes optimal classifier is some hypothesis $h^\star \in \Hcal$ with decision boundary $h^\star(x) = 0$. Note that $h^\star$ itself can have non-zero error. To measure the noise in our model we define the conditional probability distributions $\beta_L$ and $\beta_C$:
\begin{align*}
    \beta_L(x) &= \Pr[Q_L(x) = \sign(h^\star(x)) | x]\\
    \beta_C(x_1,x_2) &= \Pr[Q_C(x_1,x_2) = \sign(h^\star(x_1) - h^\star(x_2))| x_1,  x_2]
\end{align*}
Note that for all the noise models discussed below, querying $Q_L$ on the same point again (and similarly querying $Q_C$ with the same pair of points again) would generate the same answer. This is a realistic model for the case where the oracle is a human expert who may err with some probability across different inputs, but will always return the same answer on the same input.

\paragraph{Massart Noise} Massart, or bounded noise, is a well studied model of noise throughout statistics and learning theory \cite{massart2006risk,balcan2013active,xu2017noise}. Massart noise is a tractable and realistic generalization of the standard random classification noise model \cite{angluin1988learning}, where the oracle flips its response with probability $p<1/2$. Similar to \cite{awasthi2015efficient,xu2017noise}, we say ``noisy'' oracles $Q_L$ and $Q_C$ satisfy Massart noise with parameter $\np>0$ if the conditional label and comparison distributions are such that 
\begin{align*}
    \beta_L(x) &\geq \frac{1}{2} + \np ~\text{for all}~ x \in X\\
    \beta_C(x_1,x_2) &\geq \frac{1}{2} + \np ~\text{for all}~ x_1, x_2 \in X
\end{align*}
Equivalently, we say that $Q_L$ (resp. $Q_C$) satisfies Massart noise with parameter $\np$, if an adversary constructs $Q_L$ (resp. $Q_C$) by first taking the ``clean'' oracle $\bar Q_L$ (resp. $\bar Q_C$) and then flipping the result of the oracle with probability at most $\frac{1}{2} - \np$.

\paragraph{Tsybakov Low Noise Condition}
Massart error is restrictive in that the distributions $\beta_L$ and $\beta_C$ are bounded away from $\frac{1}{2}$ -- in reality, this may not be the case as examples approach the decision boundary. The Tsybakov Low Noise Condition (TNC) \cite{mammen1999smooth} offers an alternative: the closer an example is to the decision boundary, the closer its error to $1/2$. There is a natural extension of this intuition to comparison queries as well: comparisons made between arbitrarily close points should be arbitrarily noisy. A number of variants of TNC have been studied in the literature. Here we will follow the variant studied in \cite{castro2006upper,ramdas2013optimal}. Let $h^\star$ be the Bayes optimal classifier. We say $Q_L$ satisfies the Tsybakov Low Noise Condition with parameters, $m<M,$ $\varepsilon_0>0$, and $\kappa \geq 1$ $(\tnc(m,M,\kappa,\varepsilon_0))$ if $\forall x$:
\begin{align*}
    &\text{if} \ |h^\star(x)| \leq \varepsilon_0: &\frac{1}{2} + m|h^\star(x)|^{\kappa - 1}\leq \beta_L(x) &\leq \frac{1}{2} + M|h^\star(x)|^{\kappa - 1}\\
    &\text{else:}&\beta_L(x) &\geq \frac{1}{2} + m\varepsilon^{\kappa - 1}_0.
\end{align*} 
In other words, far away from the decision boundary $\beta_L(x)$ satisfies Massart noise, but approaches $1/2$ at a polynomial rate as $x$ approaches the decision boundary.
Similarly, $Q_C$ satisfies the Tsybakov Low Noise Condition with parameters, $m<M,$ $\varepsilon_0>0$, and $\kappa \geq 1$ $(\tnc(m,M,\kappa,\varepsilon_0))$ if $\forall x_1,x_2$:
\begin{align*}
    &\text{if} \ |h^\star(x_1)-h^\star(x_2)| \leq \varepsilon_0: &  \frac{1}{2} + m|h^\star(x_1) - h^\star(x_2)|^{\kappa - 1}\leq  \beta_C(x_1,x_2) &\leq \frac{1}{2} + M|h^\star(x_1) - h^\star(x_2)|^{\kappa - 1}\\
    &\text{else:} &\beta_C(x_1,x_2) &\geq \frac{1}{2} + m\varepsilon_0^{\kappa-1}.
\end{align*} 
Similar to the label case, $\beta_C(x_1,x_2)$ satisfies Massart noise for pairs of points $x_1,x_2$ which differ greatly with respect to $h^*$ and approaches 1/2 at a polynomial rate as $h^*(x_1)-h^*(x_2)$ approaches $0$.
\paragraph{Generalized Tsybakov Low Noise Condition} The Tsybakov Low Noise Condition upper and lower bounds correctness by a particular function of distance. We will consider the direct generalization of this model where these bounds are replaced with arbitrary monotone increasing functions $g_L \leq g_U: [0,\varepsilon_0] \to [0,1/2]$. We say $Q_L$ satisfies the Generalized Tsybakov Low Noise Condition with parameters $(g_L,g_U,\varepsilon_0)$ if $\forall x$:
\begin{align}
\label{eqn:gtnc-label-1}
    &\text{if} \ |h^\star(x)| \leq \varepsilon_0:&\frac{1}{2} + g_L(|h^\star(x)|) \leq \beta_L(x) &\leq \frac{1}{2} + g_U(|h^\star(x)|)\\
    &\text{else:} &\beta_L(x) &\geq \frac{1}{2} + g_L(\varepsilon_0)
\end{align}
Similarly, we say $Q_C$ satisfies the Generalized Tsybakov Low Noise Condition with parameters $(g_L,g_U,\varepsilon_0)$ if $\forall x_1,x_2$:
\begin{align}
\label{eqn:gtnc-1}
    &\text{if} \ |h^\star(x_1)-h^\star(x_2)| \leq \varepsilon_0: &\frac{1}{2} + g_L(|h^\star(x_1) - h^\star(x_2)|) \leq \beta_C(x_1,x_2) &\leq \frac{1}{2} + g_U(|h^\star(x_1) - h^\star(x_2)|)\\
    &\text{else:} &\beta_C(x_1,x_2) &\geq \frac{1}{2} + g_L(\varepsilon_0)
\end{align}
For notational convenience, we will sometimes write $g_L(x)=g_L(\varepsilon_0)$ for $x>\varepsilon_0$. In addition, since we will often need to compose $g_L$ and $g_U^{-1}$, we will use the simplified notation: 
\[
G_c(x)=g_U^{-1}\left( \frac{g_L(x)}{c}\right )
\]
where $c$ is some constant.

\subsubsection{ARPU-Learning}
\label{section:arpu}
RPU learning suffers from an inability to deal with noise. We introduce the learning framework \textit{Almost Reliable and Probably Useful Learning} (ARPU-Learning), a relaxation of RPU-learning that allows for noise, but keeps stronger reliability guarantees than PAC-learning. Recall that given a distribution $D_L$ over $X\times Y$, for a reliable learner $A$ and sample $S$, we define the loss of $A(S)$ as the measure of unlearned samples:
\begin{equation*}
    L_{D_L}(A(S)) \triangleq \mathbb{E}_{(x,y)\sim D_L}[A(S)(x) = \bot].
\end{equation*} We will commonly refer to $1-L_{D_L}(A(S))$ as the \textit{coverage} of $A(S)$. A model is a pair $(\mathcal{Q},\mathcal D_X)$ where $\mathcal{Q}$ is a set of oracles $(Q_L,Q_C)$ and $\mathcal D_X$ is a set of distributions over $X$. In ARPU-Learning, given a hypothesis class $(X, \Hcal)$ and a model $(\mathcal{Q}, \mathcal D_X)$, an adversary chooses a distribution $D_X$ from $\mathcal D_X$ and the \q{noisy} oracles $(Q_L, Q_C)$ from $\mathcal{Q}$, which induces a distribution $\tilde D_L$ over $X\times Y$ given by:\[
  \Pbb_{\tilde D_L}(x,y) = \Pbb_{D_X}(x) \Pbb(Q_L(x)= y|x).  
\] 
\begin{definition}[ARPU-Learnable] 
\label{def:arpu-learnable}
We say that a hypothesis class $(X,\Hcal)$ is ARPU-learnable under model $(\mathcal{Q}, \mathcal D_X)$ if $\forall \delta_r, \delta_u, \varepsilon > 0$, there exists a learner $A$ which is 
\begin{enumerate}
    \item Probably useful: with high probability, the learner will have large coverage: \begin{align}
    \Pr_{S\sim D_X^n} [L_{\tilde D_L}(A(S)) < \varepsilon] \geq 1- \delta_u,
    \label{eqn:useful}
\end{align} 
\item Reliable: with high probability, the learner will not make a mistake:
\begin{align}
        \Pr_{S\sim {D}_X^n}[\forall x \in X, ~ A(S)(x) \in \{ h^\star(x), \bot\}] \geq 1-\delta_r.
        \label{eqn:reliable}
\end{align} where $h^\star$ is the Bayes optimal classifier and $n=n(\varepsilon,\delta_r,\delta_u)$ is $\text{poly}\left(\frac{1}{\varepsilon},\frac{1}{\delta_r},\frac{1}{\delta_u}\right)$.
\end{enumerate}
\end{definition}
Note that in both Equations \eqref{eqn:useful} and \eqref{eqn:reliable} the probability is over the randomness of the algorithm, sample $S$, and noisy oracles $Q_L$, $Q_C$ chosen by the adversary. Also, in comparison to PAC learning, all point which are not labeled \q{$\bot$} by an ARPU-learner are labeled correctly with high probability and setting $\delta_r=0$ reduces exactly to RPU learning. Sample complexity and learnability are then defined equivalently to PAC-learning. Finally, we will refer to learners that satisfy condition \eqref{eqn:useful} as $\delta_u$-useful, and learners that satisfy condition \eqref{eqn:reliable} as $\delta_r$-reliable.  While the logical inference technique previously used to build RPU learners \cite{kane2017active,AID} are very sensitive to noise, we show in later sections how to modify those techniques to build ARPU-learners. 

\subsubsection{Passive vs Active learning}
PAC-learning traditionally is applied to supervised learning, where the learning algorithm receives pre-labeled samples. We call this paradigm passive learning. In contrast, active learning refers to the case where the learner receives unlabeled samples and may adaptively query a labeling or comparison oracle. Similar to the passive case, for active learning we study the query complexity as the minimum number of queries to learn some pair $(X, \Hcal)$ in either the PAC, RPU or ARPU-learning model. In general, passive learners learn concept classes up to error $\varepsilon$ in $\Theta(1/\varepsilon)$ samples. We add to a long line of work \cite{castro2006upper,balcan2013active,awasthi2015efficient,s-concave,kane2017active,kane2018generalized} showing that active learning can achieve such learning in only $\text{polylog}(1/\varepsilon)$ queries on important concept classes.

\subsection{Our Results}\label{intro-results}
In this work, we study ARPU-learning (Section \ref{section:arpu}) under two widely studied noise models: Massart Noise and the Generalized Tsybakov Low Noise Condition.

\subsubsection{Notation}\label{sec:notation}
We use notation where $X$ is the instance space, $\Hcal$ is the set of hypothesis from $X\to \R$, $H_d$ is the class of linear separators in $\R^d$ (corresponding to affine functions $h:\R^d \to \R$), and $H_{d,\gamma}$ is the class of linear separators in $\R^d$ with margin $\gamma$ from $X$. Since previous work \cite{balcan2013active,awasthi2015efficient} refers to the class of homogeneous linear separators as simply ``linear separators,'' we will often refer to $H_d$ as ``non-homogeneous linear separators'' to differentiate our results. For noise models, $M(\np)$ is the set of all oracles which satisfy Massart noise with parameter $\np$, $\gtnc(g_L,g_U,\varepsilon_0)$ is the set of all oracles which satisfy the Generalized Tsybakov Low Noise Condition with parameters $(g_L,g_U,\varepsilon_0)$, and $\tnc(m,M,\kappa,\varepsilon_0)$ is the set of all oracles satisfying the Tsybakov Low Noise Condition with parameters ($m,M,\kappa,\varepsilon_0$). A model is a pair $(\mathcal{Q},\mathcal D_X)$ where $\mathcal{Q}$ is a set of oracles $(Q_L,Q_C)$ and $\mathcal D_X$ is a set of distributions  over $X$. For distributions over instance space $X$ or $\R^d$, \begin{enumerate}
    \item $\mathcal C_X$ is the class of all continuous distributions over $X$,
    \item $\mathcal{LC}_d$ is the class of all log-concave distribution on $\mathbb{R}^d$,
    \item $\mathcal{SC}_d$ is the class of all s-concave distributions on $\mathbb{R}^d$ for $s\geq -\frac{1}{2d+3}$,
    \item $\mathcal{ISC}_d$ is the class of all isotropic s-concave distributions on $\mathbb{R}^d$ for $s\geq -\frac{1}{2d+3}$,
    \item $\mathcal{ACC}_{d,c_1,c_2}$ is the class of all continuous distributions $D$ which satisfy the following concentration and anti-concentration inequalities:
\begin{enumerate}
\item $\forall \alpha > 0$, $Pr_{x \sim D}[||x|| > d\alpha] \leq \frac{c_1}{\alpha}$
\item $\forall \alpha > 0, v \in \R^d, \|v\|=1, b\in \R$, $Pr_{x \sim D}[|\langle x,v \rangle + b| \leq \alpha] \leq c_2\alpha$
\end{enumerate}
    \item For hypothesis class $(X,\Hcal)$, $\mathcal A_{(X,\Hcal),a,f(d)}$ is the class of all continuous distributions $D_X$ over $X$ such that $(D_X, X, \Hcal)$ has average inference dimension $g(n) \leq 2^{-\Omega \left(\frac{n^{1+a}}{f(d)}\right)}$.
\end{enumerate}
We will call an algorithm sample (respectively time) efficient if it uses $\poly(d,\frac{1}{\varepsilon},\frac{1}{\delta_r},\frac{1}{\delta_u})$ samples (respectively time), and query efficient if it uses $\poly(d,\log \frac{1}{\varepsilon}, \log \frac{1}{\delta_r}, \log \frac{1}{\delta_u} )$ queries. Finally, for some parameter $n$ (e.g. dimension, error) and function $f: \mathbb{R} \to \mathbb{R}$, for the sake of readability we will often use the notation $\tilde{\bigo}(f(n))$ to ignore multiplicative factors that are logarithmic in $f(n)$.
\subsubsection{Massart Noise}
To begin, we show that under the Massart noise model, finite inference dimension (Definition \ref{def:ID}) implies computationally efficient ARPU-learning with exponentially better query complexity than any passive PAC-learner\footnote{Computational efficiency holds for $\lambda^{-1} = \tilde{\bigo}(\log^{1/5}(1/\varepsilon))$, query efficiency for $\lambda^{-1} = \polylog(1/\varepsilon)$.}. Recall $M(\np)$ is the set of all oracles which satisfy Massart noise with parameter $\np$, $\mathcal C_X$ is the class of all continuous distributions over $X$, and a model is a pair $(\mathcal{Q},\mathcal D_X)$ where $\mathcal{Q}$ is a set of oracles $(Q_L,Q_C)$ and $\mathcal D_X$ is a set of distributions over $X$. Note that in the ARPU-Learning model (Definition \ref{def:arpu-learnable}), given a hypothesis class $(X, \Hcal)$ and a model $(\mathcal{Q}, \mathcal D_X)$, an adversary chooses a distribution $D_X$ from $\mathcal D_X$ and the \q{noisy} oracles $(Q_L, Q_C)$ from $\mathcal{Q}$.
\begin{theorem}[Finite Inference Dimension $\implies$ ARPU-Learning under Massart Noise]
\label{intro:massart}
Let the hypothesis class $(X,\Hcal)$, $X \subseteq \mathbb{R}^d$, have inference dimension $k$ with respect to comparison queries. Then, $(X,\Hcal)$ is ARPU-learnable under model $(M(\np), \mathcal C_X)$ in time $\poly\left(d,k,\frac{1}{\delta_r},\frac{1}{\varepsilon},\log(\frac{1}{\delta_u})\right)^{\tilde{O}\left (\frac{1}{\lambda^5} \right )}$, uses only $\poly\left(k,\frac{1}{\np},\frac{1}{\varepsilon},\log(\frac{1}{\delta_r}),\log(\frac{1}{\delta_u}))\right)$ unlabeled samples, and has a query complexity of
\begin{align*}
    q(\varepsilon,\delta_r,\delta_u) = \tilde \bigo \left(k \frac{1}{\np^{10}} \log\frac{1}{\varepsilon} \log^2 \frac{1}{\delta_r} \log \frac{1}{\delta_u}\right)
\end{align*}
for $\delta_r \leq 1/2$.
\end{theorem}
To put this result into context, we note two lower bounds which together with Theorem \ref{intro:massart} show a separation between passive and active learning, and label only and comparison based ARPU-learning. In the case of passive, comparison based PAC-learning, we recall the  $\Omega\left( \frac{1}{\varepsilon} \right)$ lower bound from \cite{AID}. For label only APRU-learning, we present a lower bound novel to this work:
\begin{lemma}\label{intro:Massart:lb}
The query complexity of $1/4$-reliably, $1/8$-usefully ARPU learning $(S^1,H_2)$ with $1/2$-coverage under model $(M(\np), \mathcal C_X)$ is infinite:
\[
q(1/2,1/4,1/8) = \infty
\]
\end{lemma}
Together, these bounds show that comparison based active learning provides not only an exponential improvement in query complexity over any passive PAC-learner, but also an infinite improvement over any active ARPU-learner using only labels. Further, Theorem~\ref{intro:massart} provides the first algorithm for learning noisy non-homogeneous linear separators in two dimensions which is time, sample, and query efficient in the sense of Section~\ref{sec:notation}, since the inference dimension of $(\mathbb{R}_2,H_2)$ is 5 \cite{kane2017active}. If the instance space has bounded bit-complexity or minimal-ratio, the result also implies an efficient learner for higher dimensional non-homogeneous linear separators. 
\\
\\
Bounded bit-complexity and minimal-ratio, however, are assumptions that may not hold on real-world data. Instead, we will take a path inspired by the recent explosion of work in data science \cite{chapelle2009semi} that focuses on weakly restricting the distribution over data to beat lower bounds based off of improbable adversarial examples. While inference dimension itself is not applicable in this scenario, we will employ its average case variant, average inference dimension (Definition \ref{def:AID}). In particular, we provide a computationally efficient algorithm for learning under Massart noise under the assumption that the hypothesis class and distribution have super-exponential average inference dimension, a fact true for non-homogeneous linear separators and comparison queries across a wide range of distributions \cite{AID}. Given a hypothesis class $(X,\Hcal)$, recall $\mathcal A_{(X,\Hcal),a,f(d)}$ is the class of all continuous distributions $D_X$ over $X$ such that $(D_X, X, \Hcal)$ has average inference dimension $g(n) \leq 2^{-\Omega \left(\frac{n^{1+a}}{f(d)} \right)}$ for some $a>0$ and function of dimension $f(d)$. Then,
\begin{theorem}[Average Inference Dimension $\implies$ ARPU-Learning under Massart Noise]
\label{intro:massart:aid}
Consider any hypothesis class $(X,\Hcal)$, $X \subseteq \mathbb{R}^d$, and corresponding class of distributions $\mathcal A_{(X,\Hcal),a,f(d)}$. Then, $(X,\Hcal)$ is ARPU-learnable under model $(M(\np), \mathcal A_{(X,\Hcal),a,f(d)})$ in time $\poly\left(f(d), \frac{1}{\delta_r}, \frac{1}{\varepsilon}, \log(\frac{1}{\delta_u})\right)^{\tilde{O}\left (\frac{1}{\lambda^5} \right )}$, uses only $\poly\left(f(d),\frac{1}{\np},\log(\frac{1}{\varepsilon}),\log(\frac{1}{\delta_r}),\log(\frac{1}{\delta_u}))\right)$ unlabeled samples, and has a query complexity of
\[
    q(\varepsilon,\delta_r,\delta_u) = \tilde \bigo \left( \frac{f(d)^{1/a}}{\np^{10}} \log^{2+1/a}\frac{1}{\varepsilon} \log^2 \frac{1}{\delta_r} \log \frac{1}{\delta_u}\right),
\]
for small enough $\delta_r$.
\end{theorem}
To see the applicability of Theorem~\ref{massart:aid}, we note that Hopkins et al.\ \cite{AID} proved that a wide range of distributions lie in $\mathcal A_{(\mathbb{R}^d,H_d),1,d\log(d)}$. In particular, following \cite{AID}, we say two distributions $D$, $D'$ over $\R^d$ are affinely equivalent if there is an invertible affine map $f\colon \R^d \to \R^d$ such that $D(x) = D'(f(x))$. Hopkins et al.\ \cite{AID} proved that distributions which may be affinely transformed to a distribution with anti-concentration and concentration (i.e.\ to a distribution in $\mathcal{ACC}_{d,c_1,c_2}$) lie in $\mathcal A_{(\mathbb{R}^d,H_d),1,d\log(d)}$, a condition satisfied by s-concave distributions\footnote{As noted in \cite{AID}, s-concavity needs $\alpha>16$ in condition 1, but this does not affect our proofs.}. Recall that $\mathcal{SC}_d$ is the class of all s-concave distribution, $s\geq -\frac{1}{2d+3}$, on $\mathbb{R}^d$ and $H_d$ is the class of both homogeneous and non-homogeneous linear separators in $\R^d$. Then, as a direct corollary to Theorem~\ref{intro:massart:aid}, we have
\begin{corollary}
The hypothesis class $(\R^d,H_d)$ is ARPU-learnable under model $(M(\np), \mathcal{SC}_d)$ in time
\\
$\poly\left(d, \frac{1}{\delta_r}, \frac{1}{\varepsilon}, \log(\frac{1}{\delta_u})\right)^{\tilde{O}\left (\frac{1}{\lambda^5} \right )}$, uses only $\poly\left(d,\frac{1}{\np},\log(\frac{1}{\varepsilon}),\log(\frac{1}{\delta_r}),\log(\frac{1}{\delta_u}))\right)$ unlabeled samples, and has a query complexity of:
\[
    q(\varepsilon,\delta_r,\delta_u) = \tilde \bigo \left( d\frac{1}{\np^{10}} \log^{3}\frac{1}{\varepsilon} \log^2 \frac{1}{\delta_r} \log \frac{1}{\delta_u}\right).
\] for small enough $\delta_r$.
\end{corollary}
Previous work showed a similar result for homogeneous linear separators over nearly isotropic log-concave distributions \cite{awasthi2016learning} and isotropic s-concave distributions \cite{s-concave} with label queries. However, their techniques cannot be extended to the non-homogeneous case due to a $\poly(\frac{1}{\varepsilon})$ lower bound on the query complexity of active label-only learners \cite{AID}. Thus it is only by leveraging the additional power of comparison queries that we extend efficient learning to non-homogeneous linear separators over (not necessarily isotropic) s-concave distributions.
\subsubsection{Generalized Tsybakov Low Noise Condition}
While Massart noise is a clean theoretical model, its assumption that the noise is bounded away from $1/2$ is not necessarily reminiscent of the real world. This motivates us to study a variant of the Tsybakov Low Noise Condition, a model in which noise is unbounded as data approaches the Bayes optimal classifier. However, learning in this unbounded regime is harder, as evidenced by the polynomial query lower bounds of \cite{hanneke2015minimax,wang2016noise,xu2017noise}. In order to ARPU-learn in this regime, we need to introduce several restrictions not present for our Massart algorithms. First, instead of allowing any hypothesis class with finite inference dimension, we will only consider (non-homogeneous) linear separators. Second, we will either assume some margin $\gamma$, or that the distribution satisfies certain weak concentration and anti-concentration bounds. To begin, we consider learning hypothesis classes over any continuous distribution with finite inference dimension and margin. Recall $\gtnc(g_L,g_U,\varepsilon_0)$ is the set of all oracles which satisfy the Generalized Tsybakov Low Noise Condition with parameters $(g_L,g_U,\varepsilon_0)$, $H_{d,\gamma}$ is the class of linear separators in $\R^d$ with margin $\gamma$ from $X$, and $\mathcal C_X$ is the class of all continuous distributions over $X$.
\begin{theorem}[Finite Inference Dimension and Margin $\implies$ ARPU-Learning under GTNC]
\label{intro:TNC}
Let $X\subseteq \mathbb{R}^d$ and $(X,H_{d,\gamma})$ have inference dimension $k$ with respect to comparison queries. Then for small enough $\delta_r$, $(X,H_{d,\gamma})$ is ARPU-learnable under model $(\gtnc(g_L,g_U,\varepsilon_0), \mathcal C_X)$ with query complexity:
\begin{align*}
q(\varepsilon,\delta_r,\delta_u) &=\tilde{\bigo} \left (\frac{k^{10}}{\left(g_L\circ G_8 \circ \frac{G_4(\gamma')}{2}\right)^{14}}d\log^{2}\left( \frac{1}{\delta_r} \right)\log\left(\frac{1}{\varepsilon}\right)\log\left(\frac{1}{\delta_u}\right)\right ).\\
\end{align*}
Where
\begin{align*}
\gamma' = \min\left(\frac{\gamma}{2d},\frac{\varepsilon_0}{2}\right), G_c(x) = g_U^{-1}\left( \frac{g_L(x)}{c}\right)
\end{align*}
\end{theorem} 
We prove in addition that while ARPU-learning may no longer be impossible using only labels when margin is introduced, it still suffers from query inefficiency due to the curse of dimensionality.
\begin{lemma}\label{intro:TNC:margin:lb}
Let $X \in \mathbb{R}^d$ be the $d$-dimensional hypercube $\{0,1\}^d$ modified to have a ball of radius $\frac{1}{4\sqrt{d}}$ centered about each point. The query complexity of ARPU-learning $(X,H_{d,\frac{1}{4\sqrt{d}}})$ under model $(\gtnc(g_L,g_U,\frac{1}{4\sqrt{d}}), \mathcal C_X)$ is at least:
\[
q(1/4,1/8,1/16) \geq 2^{d-1}
\]
\end{lemma}
In the above example, $(X,H_{d,\frac{1}{4\sqrt{d}}})$ has inference dimension $\tilde{\bigo}(d)$ by a minimal-ratio argument from \cite{kane2017active}. Theorem~\ref{intro:massart} thus gives an algorithm using only $\poly(d)$ queries, demonstrating the exponential gap in query complexity between label only and comparison based ARPU-learning with Tsybakov noise. Due to margin causing bounded error in label queries, another way to view this result is the statement that comparison queries with unbounded error exponenentially improve the query complexity of ARPU-learning using only labels with bounded error.
\\
\\
Similar to the case of Massart noise, we may drop the restrictive assumptions of finite inference dimension and margin by assuming weak distributional requirements. Unlike in the case of Massart, here we deal with the requirements directly rather than assuming average inference dimension. Recall $\mathcal{ACC}_{d,c_1,c_2}$ is the class of all continuous distributions $D$ with the following properties:
\begin{enumerate}
\item $\forall \alpha > 0$, $Pr_{x \sim D}[||x|| > d\alpha] \leq \frac{c_1}{\alpha}$
\item $\forall \alpha > 0, v \in \R^d, \|v\|=1, b\in \R^d$, $Pr_{x \sim D}[|\langle x,v \rangle + b| \leq \alpha] \leq c_2\alpha$
\end{enumerate}

\begin{theorem}[Concentration and Anti-Concentration $\implies$ ARPU-learning under $\gtnc$]\label{intro:TNC:aid}
For small enough $\delta_r$, the hypothesis class $(\R^d,H_d)$ is ARPU-learnable under model $(\gtnc(g_L,g_U,\varepsilon_0), \mathcal{ACC}_{d,c_1,c_2})$ with query complexity:
\[
q(\varepsilon,\delta_r,\delta_u) =\tilde{\bigo} \left (\frac{d^{11}}{\left(g_L \circ G_8 \circ \frac{G_2\circ\frac{G_4(\varepsilon')}{4d}}{2}\right)^{14}}
\log^{2}\left( \frac{1}{\delta_r} \right)\log\left(\frac{1}{\delta_u}\right)\log^2\left(\frac{1}{\varepsilon}\right)\right ),
\]
where
\begin{align*}
\varepsilon' = \min\left(\frac{\varepsilon}{4c_2},\frac{\varepsilon_0}{2}\right).
\end{align*}
\end{theorem}
Since isotropic s-concave distributions satisfy these conditions \cite{s-concave,AID}, we get the immediate corollary for TNC noise under isotropic s-concave distributions. Recall that $\mathcal{ISC}_d$ is the class of all isotropic ($0$ mean, identity variance) s-concave distribution on $\mathbb{R}^d$, and $H_d$ is the class of non-homogeneous linear separators in $\R^d$.
\begin{corollary}\label{intro:dd:TNC}
The hypothesis class $(\R^d,H_d)$ is ARPU-learnable under model $(\tnc(m,M,\kappa,\varepsilon_0), \mathcal{ISC}_d)$ with query complexity:
\begin{align*}
q(\varepsilon,\delta_r,\delta_u) &=\tilde{\bigo} \left( \frac{2^{14\kappa} M^{42} d^{14 \kappa-3}}{m^{56}\varepsilon'^{14(\kappa-1)}}\log^{2}\left( \frac{1}{\delta_r} \right)\log\left (\frac{1}{\delta_u} \right)\right ).
\end{align*}
Where
\begin{align*}
\varepsilon' = \min\left(\frac{\varepsilon}{16},\frac{\varepsilon_0}{2}\right).
\end{align*}
\end{corollary}
This result similarly extends previous work on homogeneous linear separators over isotropic log-concave distributions \cite{balcan2013active,wang2016noise} to the non-homogeneous case. In comparison to Hanneke and Yang's \cite{hanneke2015minimax} distribution free algorithm for label only PAC-learning, Corollary~\ref{intro:dd:TNC} provides an improved query complexity for $1 < \kappa < \frac{15}{14}$, and more importantly provides the reliability guarantees of the ARPU-learning model.
\\
\\
Finally, note that unlike Theorems~\ref{intro:massart}, \ref{intro:massart:aid}, and \ref{intro:TNC}, Corollary~\ref{intro:dd:TNC} has polynomial rather than polylogarithmic dependence on $\varepsilon^{-1}$. This is unavoidable, as we prove a lower bound also polynomial in $\varepsilon^{-1}$.
\begin{lemma}\label{intro:dd:TNC-lower-bound}
	The query complexity of actively PAC-learning $(\R^2, H_2)$ under model $(\tnc(m,M,\kappa,\varepsilon_0), \mathcal{SC}_2)$ is at least 
	\[
	q(\varepsilon, 1/8) = \Omega\left(\frac{1}{\max\{\epsilon, \epsilon^{\kappa - 1}\}}\right)
	\] 
	where $\varepsilon \leq \frac{\left(\frac{1}{16m}\right)^{\frac{1}{\kappa-1}}}{4}$.
\end{lemma}
Thus the main advantage of comparisons in this regime is their added reliability.
\subsection{Techniques}\label{techniques}
\subsubsection{Inference Dimension}\label{section:ID}
Our algorithms will follow the form of the learning technique for hypothesis classes with finite inference dimension (Definition \ref{def:ID}) introduced in \cite{kane2017active}. Drawing and querying a subsample $S$, Kane {et al.} build a weak learner by defining a Linear Program (LP) with constraints given by the query responses $Q(S)$, and objective function defined by the input point to be labeled. Through a symmetry argument, Kane {et al.} are able to show that if $S$ is large enough with respect to the inference dimension, the coverage of this weak learner will be at least $3/4$. Since we will rely on this argument throughout our paper, we offer a brief description here. 
\\
\\
The expected coverage of the learner may be viewed as the probability that a randomly drawn point from the distribution is inferred by the LP. Since our weak learner is built from some finite sample from the same distribution, symmetry gives that this is equivalent to the probability that any of $|S|+1$ points can be inferred from the other $|S|$. Kane {et al.} then provide the following observation for $|S|=n$ and inference dimension $k$ which proves that setting $n=4k$ gives coverage at least $3/4$.
\begin{observation}[Observation 3.4 \cite{kane2017active}]
\label{ID:total}
Let the hypothesis class $(S,\Hcal)$, $|S|=n$, have inference dimension $k$ for the set of binary queries $Q$. Then $\forall h \in \Hcal$, there exists a subset $S' \subset S$ of size $n-k+1$ such that $\forall x \in S'$:
\[
Q(S - \{x\}) \yields{h} x
\]
\end{observation}
Inference dimension on its own, however, is restrictive. Using only comparisons and labels, the inference dimension of linear separators in three or more dimensions is infinite, which implies the existence of realizable distributions with $\Omega(\frac{1}{\varepsilon})$ query complexity \cite{kane2017active}. To get around this barrier, we will introduce weak distributional assumptions and instead employ the framework of average inference dimension introduced in \cite{AID}. Average inference dimension (Definition~\ref{def:AID}) allows us to build algorithms for hypothesis classes with infinite inference dimension, as long as the distribution it is over is sufficiently nice. We will take advantage of a reduction from average to worst case inference dimension to prove such results:
\begin{observation}[Observation 3.6 \cite{AID}]
\label{avg-worst}
Let $(D, X, \Hcal)$ have average inference dimension $g(n)$, and $S \sim D^n$. Then $(S, \Hcal)$ has inference dimension $k$ with probability:
\[
\Pr[\text{inference dimension of}~ (S,\Hcal) \leq k] \geq 1-{ \binom{n}{k}}g(k).
\]
\end{observation}
\subsubsection{Noisy Sorting} \label{sec:noisy-sorting}
The linear program used as a weak learner relies heavily on the correctness of $Q(S)$, making noisy oracles a challenging problem. To retain correctness and reliability, we rely on using extra points outside of the linear program to help identify the true answers $Q(S)$. This idea is not all together new. Contemporaneously with Kane et al., Xu, Zhang, Singh, Miller and Dubrawski \cite{xu2017noise} suggested using noisy comparisons as a sub-routine in older active learning algorithms to correct for noise in labels. However, as Xu {et al.} \cite{xu2017noise} point out, this technique does not work for Kane et al.'s \cite{kane2017active} algorithm which requires corrected comparisons as well. Instead, we adopt and adapt a noisy sorting algorithm from Braverman and Mossel \cite{braverman2009sorting}. 
\\
\\
Braverman and Mossel \cite{braverman2009sorting} study the problem of recovering the best possible ranking from an ordered set with access to a noisy comparison oracle $Q_C$. In particular, given a ground set $S$ of size $n$, Braverman and Mossel aim to find an order $\pi$ that minimizes the number of discrepancies with the measured comparisons $Q_C(S)$, denoted by the order relation $\widetilde{<}$:
\[
\underset{\pi}{\arg\min} ~ |\{x_i,x_j \in S : (\pi(x_i)<\pi(x_j)) \land (x_j \widetilde{<} x_i)\}|.
\]
If the oracle $Q_C$ flips comparisons with probability exactly $p<1/2$ and the true ordering has a uniform prior, Braverman and Mossel \cite{braverman2009sorting} note that this scoring function has a nice probabilistic interpretation: it is a Maximum Likelihood ordering
\[
\text{arg}\max_{\pi \in S_n} P(\pi | Q_C(S)).
\]
Braverman and Mossel \cite{braverman2009sorting} call finding such an ordering the Noisy Signal Aggregation (NSA) problem, and provide a randomized algorithm that uses only $\bigo_{\np}(n\log(n))$ comparisons for oracles satisfying Massart noise with parameter $\np$. Further, they provide an important structural insight into MLE orderings: with high probability, no point in an MLE order has moved further than $\bigo_{\np}(\log(n))$ from its position in the true order.
\begin{theorem}[Optimal Ranking \cite{braverman2009sorting}]
\label{order}
Let $S$ be a set of size $n$ with underlying order $1,\ldots,n$ and $\sigma$ an MLE order for $S$ under comparisons given by an oracle $Q_C$ satisfying Massart noise with parameter $\lambda$. Then with probability at least $1-\delta$:
\begin{align*}
    \max_i|\sigma(i) - i| &\leq \bigo\left(\frac{\log^2(1/\lambda)\log(n/\delta)}{\lambda^3}\right)
\end{align*}
as long as $n$ or $\frac{1}{\delta}$ is at least exponential in $\lambda^{-1}$.
\end{theorem}
This pointwise movement allows us to determine with high probability comparisons between points that are well-separated throughout an MLE order. By using only such separated points to build our inference LP, our algorithms are almost reliable -- a point can only be mislabeled if some well-separated comparison is wrong, a low probability event.
\\
\\
While Braverman and Mossel's algorithm is query efficient and has a strong pointwise movement guarantee, its exponential time complexity in the error parameter is the main limiting factor in the computational efficiency of our algorithm for Massart noise. The existence of an efficient (polynomial in error) sorting scheme that retains some sub-linear (not necessarily logarithmic) point-wise movement bound under Massart noise would immediately imply computationally and query efficient algorithms for Massart noise for $\lambda^{-1}$ poly-logarithmic in $\frac{1}{\varepsilon}$, rather than for $\lambda^{-1} = \tilde{\bigo}(\log^{1/5}(1/\varepsilon))$ as we require. Follow up works on Braverman and Massart's algorithm \cite{gavenvciak2019sorting,geissmann2018optimal,klein2011tolerant} made progress in this direction, providing algorithms with significantly improved time complexity, but only work for $\lambda$ bounded from below by some constant. Providing an algorithm that remains efficient while $\lambda$ goes to 0 is an open problem.

\subsubsection{Cluster Detection and Inference}
Braverman and Mossel's noisy sorting algorithm works well in the case of bounded error, but noise models with unbounded error require a different approach. The particular model we examine in this case, the Generalized Tsybakov Low Noise Condition, is a distance based error metric. This means that as points approach each other in function value, their comparisons ``look random''. We can use this fact to detect clusters of points close in function value by testing whether comparisons between them look like they have been drawn at random. In particular, we define a natural measure of randomness that we call equitability:
\begin{definition}[Equitability]
Let $S$ be a set with comparisons denoted by $\widetilde{<}$ on each pair of elements. For an element $x \in S$, let $v(x)$ denote the number of elements $y \in S$ such that $y \widetilde{<} x$. We call $S$ $\varepsilon$-\textit{equitable} if 
\[
\forall x \in S, \left(\frac{1}{2}-\varepsilon \right)|S| \leq v(x) \leq \left(\frac{1}{2}+\varepsilon \right)|S|
\]
\end{definition}
We prove a bi-directional equivalence between clusters and equitable sets: any cluster is equitable with high probability, and any equitable set contains a large cluster with high probability.
\\
\\
If a sample has no cluster, we will prove that a modified version of noisy sorting is sufficient to learn. On the other hand, if we detect a cluster, another approach is required. To handle this case we prove a novel structural lemma regarding the inference power of clusters, showing that any cluster of size $\Omega(d\log(d))$ must contain a point that can be inferred from the rest.



\section{Massart Noise}\label{Sec:massart}
\subsection{Lower Bounds}
In this section we provide two lower bounds: the first to separate comparisons from label only ARPU learning, and the second to explain our restriction to continuous distributions. Our label only lower bound uses the same distribution that shows an exponential gap in active PAC learning between labels and comparisons \cite{dasgupta2005analysis,kane2017active}, a circle, except in the case of ARPU learning, the gap is infinite.
\begin{lemma}[Restatement of Lemma \ref{intro:Massart:lb}]
The query complexity of $1/4$-reliably, $1/8$-usefully ARPU learning $(S^1,H_2)$ with $1/2$-coverage under model $(M(\np), \mathcal C_X)$ is infinite:
\[
q(1/2,1/4,1/8) = \infty
\]
\end{lemma}
\begin{proof}
By Yao's minimax principle it is enough to show that the adversary may pick a distribution over hyperplanes such that no learner can $1/4$-reliably and $1/8$-usefully learn with coverage $1/2$. In particular, assume that the adversary picks a uniform distribution over all tangent hyperplanes to the circle. This may be equivalently thought of as the adversary picking a single point on the circle to be negative, and the rest to be positive. Note that the probability that a learner which queries a finite number of points finds the negative point is $0$. 
\\
Let the learner fix an optimal strategy, querying whatever points they desire. With probability 1, the learner will always query the same set of points since they receive all positive labels. The learner is then left to label $1/2$ the measure of the circle blind, since all points except a measure $0$ set (the queried points) are equally likely to be the negative point. No matter which set the learner chooses, the probability that it mislabels a point is at least $1/2$, violating the ARPU-learning requirement that the learner must must label at least $1/2$ of the points with probability at least $5/8$ while making no errors.
\end{proof}
Note that this lower bound holds even in the noiseless case, which is strictly weaker than Massart as the adversary may simply choose no noise. 
\\
Second, we justify why our upper bounds are only for continuous distributions, as the inference dimension framework was initially developed for the worst case rather than distributional model. However, with the introduction of noise, we observe that learning up to arbitrary error is no longer possible over some distributions.
\begin{observation}
Let $(X,\Hcal)$ be a hypothesis class, and $D$ a distribution on $X$ whose support consists of a single point $x$. Let the corresponding noisy label and comparison oracles $(Q_L,Q_C) \in M(\np)$. If there exist $h,h' \in \Hcal$ s.t. $h(x)\neq h'(x)$, then no learner can correctly label $x$ with probability more than $1/2+\np$. 
\end{observation}
This lower bound holds as well across a wide range of distributions containing points with non-zero measure. Take, as an example, a distribution which samples uniformly from the unit ball with probability $1/2$, and some disjoint point $x$ with probability $1/2$. Setting the error parameter low enough would force the learner to correctly label $x$, and since the adversary can pick a classifier such that the point cannot be inferred from comparisons, a similar lower bound holds. In order to avoid such examples, we will restrict our consideration to continuous distributions.
\subsection{Finite Inference Dimension}
With the lower bound out of the way, we prove that hypothesis classes with finite inference dimension are efficiently ARPU-learnable under Massart noise. Recall $M(\np)$ is the set of all oracles which satisfy Massart noise with parameter $\np$, $\mathcal C_X$ is the class of all continuous distributions over $X$ and a model is a pair $(Z,\mathcal D_X)$ where $Z$ is a set of oracles $(Q_L,Q_C)$ and $\mathcal D_X$ is a set of distributions over $X$. Note that in the ARPU-Learning model (Definition \ref{def:arpu-learnable}), given a hypothesis class $(X, \Hcal)$ and a model $(Z, \mathcal D_X)$, an adversary chooses a distribution $D_X$ from $\mathcal D_X$ and the \q{noisy} oracles $(Q_L, Q_C)$ from $Z$. The following is our learning algorithm.
\\
\\
\begin{algorithm}[H]
\SetAlgoLined
\KwResult{Returns an $\delta_u$-useful, $\delta_r$-reliable learner with coverage $1-\varepsilon$ for model ($M(\np), \mathcal{C}_X$)}
\nonl \textbf{Input:} Noisy oracles $Q_L, Q_C \in M(\np)$, unknown distribution $D_X \in \mathcal{C}_X$\\
\nonl \textbf{Parameters:} 
\begin{itemize}
    \item Inference dimension $k$
    \item Iteration cap $T=\poly(\log(1/\delta_u),\log(1/\varepsilon))$
    \item  Time cap $T_{sort}=\poly(n,1/\delta_r)^{\tilde{\bigo}(\np^{-5})}$
    \item Query cap $Q_{sort}=\poly(k,\np^{-1},\log(1/\delta_r))$
    \item Sample cap $C=\poly(1/\varepsilon,\log(1/\delta_u),\log(k),\log(\np^{-1}),\log(\log(1/\delta_r)))$
    \item Sample sizes $n=\poly(k,\log(1/\delta_r),\np^{-1})$ and $m=\poly(k)$
    \item Set of linear program constraints $LP=\{\}$
\end{itemize}
\nonl \textbf{Algorithm:}
\begin{enumerate}[leftmargin=*]
    \item Sample $S \sim D_X^n$ restricted to points un-inferred by $LP$, and sort $S \cup \zo$ via noisy oracles $Q_L$ and $Q_C$ by the algorithm of \cite{braverman2009sorting}. If noisy sorting exceeds time threshold $T_{sort}$ or query threshold $Q_{sort}$, abort sorting.
    \item Sample $S' \sim D_X^m$ restricted to points un-inferred by $LP$ and insert into the order on $S$. Update $LP$ constraints using comparisons and labels of elements in $S'$ separated by $\bigo_{\np}(\log(n))$ from each other and from $\zo$
    \item If at any point in steps $1$ or $2$, $C$ inferred samples are drawn in a row, return the current $LP$. Repeat from step one until iteration cap $T$ is reached and return $LP$. 
\end{enumerate}
 \caption{Efficient ARPU-learning under Massart Noise}
 \label{alg:massart}
\end{algorithm}
Before proving the lemmas necessary to show the coverage of Step 2 from Algorithm~\ref{alg:massart}, we will restate our theorem of the efficient learnability of hypothesis classes with finite inference dimension under Massart noise.
\begin{theorem}[Restatement of Theorem~\ref{intro:massart}]
\label{massart}
Let the hypothesis class $(X,\Hcal)$, $X \subseteq \mathbb{R}^d$, have inference dimension $k$ with respect to comparison queries. Then, $(X,\Hcal)$ is ARPU-learnable under model $(M(\np), \mathcal C_X)$ in time $\poly(d,k, \frac{1}{\delta_r},\frac{1}{\varepsilon},\log(\frac{1}{\delta_u}))^{\tilde{O}\left (\frac{1}{\np^5} \right )}$, uses only $\poly(k,\frac{1}{\np},\frac{1}{\varepsilon},\log(\frac{1}{\delta_r}),\log(\frac{1}{\delta_u})))$ unlabeled samples, and has a query complexity of
\begin{align*}
    q(\varepsilon,\delta_r,\delta_u) = \tilde \bigo (k \frac{1}{\np^{10}} \log\frac{1}{\varepsilon} \log^2 \frac{1}{\delta_r} \log \frac{1}{\delta_u})
\end{align*}
for small enough $\delta_r$.
\end{theorem}
See Algorithm~\ref{alg:massart}. The proof of this theorem lies in the combination of Braverman and Mossel's \cite{braverman2009sorting} approximate ordering with Kane et al.'s \cite{kane2017active} inference based algorithm. 
The idea is as follows:
\paragraph{Step 1:} Draw a sample $S\sim D^n_X$, and sort it into an MLE order by the algorithm of Braverman and Mossel \cite{braverman2009sorting}. Draw another $m$ points, and independently slot them into the ordering on $S$ near their true position (again by an algorithm from \cite{braverman2009sorting}).
\paragraph{Step 2:} From the $m$ points, create a clean subset of points with correct labels and comparisons by selecting a chain of points separated by $\Omega_{\np}(\log(n))$ in the MLE order of $S$ to build an inference LP. This LP correctly infers points with high probability by \cite{braverman2009sorting}, and has large coverage due to the space's finite inference dimension \cite{kane2017active}.
\paragraph{Step 3:} Restrict $D$ (by rejection sampling) to points un-inferred by the LP in step 2, and repeat steps 1 and 2 until coverage has reached $1-\varepsilon$.
\\
\\
The main challenge of the proof then comes down to proving the correctness and coverage of Step 2. First, we need to show that points separated in $S$ by $\Omega_{\np}(\log(n))$ are correctly ordered. Since our sample size $n$ will not be exponential in $\frac{1}{\np}$, we need to slightly modify Theorem~\ref{order} for this result.
\begin{observation}[Point-wise Movement]
\label{order2}
Let $S$ be a set with underlying order $1\ldots n$. If $\sigma$ is an MLE order for $S$ under noisy label and comparison oracles $(Q_L,Q_C) \in M(\np)$, then with probability at least $1-\delta$:
\begin{align}
    \max_i|\sigma(i) - i| &\leq \bigo\left(\frac{\log^3(1/\np)\log(n/\delta)}{\np^5}\right) \label{mvt:bound}
\end{align}
as long as $n$ or $1/\delta$ is polynomial in $\np^{-1}$.
\end{observation}
\begin{proof}
Braverman and Mossel define a parameter $m_2$ during their proof as:
\[
m_2 = \bigo\left(\frac{\log(n/\delta)}{\np}\right).
\]
The requirement on size of $n$ or $1/\delta$ of Theorem \ref{order} then comes from the final equation of Lemma 28 \cite{braverman2009sorting}:
\[
m_2 = \Omega\left(\frac{\log^2(1/\np)}{\np^3}\right).
\]
Increasing $m_2$ by a factor of $\frac{\log(1/\np)}{\np^2}$ removes the need for exponential dependence on $\np^{-1}$, but increases the pointwise movement bound by the same factor.
\end{proof}
Note that points in $S$ separated by $2\max\limits_i|\sigma(i) - i|$ in an MLE order are thus correctly ordered with high probability. As a result, picking a chain of points each separated by twice Equation \eqref{mvt:bound} gives an entire set of points with correct comparisons with high probability. 
\\
\\
However, using an MLE order itself is challenging. Recall from Section~\ref{techniques} that we compute the expected coverage of our learner by the probability that it infers an additionally drawn point. If we use an MLE order, we cannot directly appeal to the symmetry argument of \cite{kane2017active}, as adding an additional point to $S$ might change the MLE order we have picked. To get around this, our learner is not built off of $S$ itself, but $S'$, a set of additional points which we place into the order on $S$ independently of each other. This independence allows us to directly appeal to the argument of \cite{kane2017active}.
\\
\\
Our method of finding a clean subset, however, is currently for $S$ -- we need to modify the method to find a subset of $S'$ with correct labels and comparisons. We do this in two steps. First, we adopt a method from \cite{braverman2009sorting} for inserting points into a previously sorted set such that they cannot be too far away from their true position. This implies that if points in $S'$ are separated by enough points in S in the underlying order on $S \cup S'$, we will be able to correctly compare them with high probability. Second, we show that because the underlying true order on $S \cup S'$ is uniform, there exists a chain of such points in $S'$ with constant probability from which we can build our cleaned set.
\begin{lemma}[Slotting \cite{braverman2009sorting}]
\label{massart:slot}
Let $S$ of size $n$ and $S'$ of size $m$ be ordered sets with noisy label and comparison oracles $(Q_L,Q_C) \in M(\np)$. Divide an MLE order $\sigma$ of $S$ into $b$ blocks $B_i$ of size at least:
\[
|B_i| \geq \Omega\left(\frac{\log^3(1/\np)\log(\frac{nm}{\delta})}{\np^5}\right).
\]
There exists an algorithm placing points in $S'$ into $\sigma$ such that, with probability at least $1-\delta$, any pair of points separated by $4$ blocks are in the correct order.
\end{lemma}
\begin{proof}
This lemma is a slight modification of part of \cite[Theorem 30]{braverman2009sorting}. Assume some $x \in S'$ lies in the $i$-th block $B_i$ in the true order. By Observation \ref{order2}, with probability at least $1-\delta$, $x$ must be bigger than all elements before $B_{i-1}$ and smaller than all elements past $B_{i+1}$. To find which side of a block $B$ $x$ lies in, we measure whether $x$ is greater than, or less than a majority of elements in the block. A standard Chernoff bound gives that the probability the majority is incorrect is at most $e^{-\np^2|B|} \leq \frac{\delta}{mn}$, and union bounding over blocks and $S'$ gives that all elements will be slotted up to an error of two blocks on either side. Note further that this slotting procedure may be performed by binary search, and thus uses at most $\bigo (\log(b)m|B|)$ queries in total. Finally, since elements must be slotted within two blocks of their true position, any pair of elements separated by at least 4 full blocks must be in the correct order.
\end{proof}
Since we can safely compare points separated by $4$ blocks in the MLE order, and points slot within $2$ blocks of their true position, points separated in the underlying order by $8$ blocks can be correctly compared with high probability. It is left to show that there is a large enough chain of points in $S'$ separated by 8 blocks in $S$.
\begin{lemma}\label{slot-bound}
Let $S$ of size $n$ and $S'$ of size $32k+16$ be ordered sets with noisy label and comparison oracles $(Q_L,Q_C) \in M(\np)$. Let the size of $S$ satisfy:
\[
n \geq \Omega\left(\frac{k\log^3(1/\np)\log(\frac{nk}{\delta})}{\np^5}\right).
\]
Then with constant probability we can find a subset of $4k$ points from $S'$ which can be labeled and compared correctly with probability $1-\delta$.
\end{lemma}
\begin{proof}
Consider the true order $\pi$ on the set $S \cup S'$. Let $\zo$ be the special point whose comparison to another point $x$ is given by $x$'s label $Q_L(x)$. Lemma~\ref{massart:slot} provides an algorithm for determining the labels and comparisons of points in $S'$ separated by more than 
\[
c = \Omega\left(\frac{\log^3(1/\np)\log(\frac{nk}{\delta})}{\np^5}\right)
\]
elements in $S$ and not within $c$ of $\zo$. Consider dividing the order $\pi$ restricted to S (denoted by $\pi_S$) up into $b=32k+16$ equal blocks $B_i$ of size at least $c$. Since any two points in $S'$ which are separated by more than a block in $\pi_S$ will be correctly ordered by Lemma~\ref{massart:slot} with probability $1-\delta$ and only $2$ non-contiguous blocks can be adjacent to $\zo$, it is sufficient to find a chain of non-contiguous blocks of size $4k+2$ that all contain a point in $S'$. To simplify this, consider the set of every other block ($\mathcal{B}_{odd} = \{B_1,B_3,...\}$), and let $Y$ be the random variable denoting the number of blocks in $\mathcal{B}_{odd}$ without a point in $S'$. To upper bound the value of $Y$, we bound its mean and variance and apply Chebyshev's inequality. Note that since $S$ and $S'$ are drawn i.i.d., the ordering on $S \cup S'$ is uniform at random. We can write Y as the sum of indicator variables $Y_1+Y_3+...$, where $Y_i$ denotes the event that $B_i$ does not have a point in $S'$. Since the ordering is uniform, the probability that a point in $S'$ lies in any given block is $\frac{1}{b}$. Using this, we can bound the expectation of $Y$ by:
\begin{align*}
    \mathbb{E}[Y]&=\sum\mathbb{E}[Y_i]\\
    &= \frac{b}{2}\left(1-\frac{1}{b}\right)^b\\
    &\leq \frac{b}{2e},
\end{align*}
and similarly the variance of $Y$ by:
\begin{align*}
    Var(Y) &= \sum\limits_{\underset{i \neq j}{i,j}} \mathbb{E}[Y_iY_j] + \sum \mathbb{E}[Y_i^2]- \sum\mathbb{E}[Y_i]^2\\
    &= \frac{b}{2}\left( \frac{b}{2} - 1 \right )\left ( 1-\frac{2}{b}\right )^b + \frac{b}{2}\left(1-\frac{1}{b}\right)^b - \left ( \frac{b}{2}\left(1-\frac{1}{b}\right)^b \right )^2\\
    &\leq \frac{b^2}{4e^2} - \frac{b^2}{32} \leq \frac{b^2}{64}.
\end{align*}
Here the second to last inequality follows from the assumption that $b\geq 48$ (or equivalently that $k\geq 1$). Noting that the number of blocks with a point from $S'$ is $\frac{b}{2} - Y$,  Chebyshev's inequality then gives that a constant fraction of the blocks must have a point from $S'$ with constant probability:
\begin{align*}
&\Pr\left[Y > \frac{3b}{8}\right] < 4/9\\
\implies &\Pr \left[\frac{b}{2}-Y > 4k+2\right] > 5/9
\end{align*}
\end{proof}
Lemma~\ref{slot-bound} allows us to build a clean set of points with correct comparisons and labels. By feeding this set of points into an inference LP, we create a weak learner that infers a constant fraction of the space with constant probability. 
\begin{lemma}[Weak Learner]
\label{lemma:weak-learner-massart}
Let $(X,\Hcal)$ have inference dimension $k$, and let the label and comparison oracles $Q_L,Q_C \in M(\np)$. Then there exists a constant $c_1>0$ such that for any $1/2>\delta_r>0$, there exists a weak learner that $3\delta_r$-reliably learns $(X,\Hcal)$, has coverage $c_1$ with probability $\geq c_1$, makes at most $q_{wl}(\delta_r)$ queries, and runs in time $\poly(k, \frac{1}{\delta_r})^{\tilde{O}(\np^{-5})}$, where
\begin{align*}
    q_{wl}(\delta_r) &= \tilde \bigo \left( \frac{k}{\np^{10}} \log^2 \frac{1}{\delta_r} \right)
\end{align*}
\end{lemma}
\begin{proof}
Let $S \sim D_X^n$ be a sample from our distribution, where
\begin{align*}
n &= \Theta\left(\frac{k}{\np^5}\log^3\frac{1}{\np}\log\frac{k}{\np \delta_r}\right) 
\end{align*} 
Following Lemmas \ref{massart:slot} and \ref{slot-bound}, we we will slot a second, i.i.d. drawn set $S'$ of points into our MLE order where $|S'| = 32k+16$. Then with constant probability we can find a subset of $4k$ points in $S'$ which may be correctly ordered and labeled with probability at least $1 - \delta_r$.
\\
\\
We are now in position to apply the symmetry argument from \cite{kane2017active} to show that this subset gives constant coverage with constant probability. The expected coverage is given by the probability that an additional, independently drawn point $x\sim D_X$ is inferred:
\[
\E[\text{Coverage}] = Pr_{(x_1,\ldots,x_{|S'|+1}) \sim D_X^{|S'|+1}}[\{x_1,\ldots,x_{|S'|}\} \yields{} x_{|S'|+1}].
\]
Since $S'$ and $x$ are drawn randomly, the right hand side is equivalent to the probability that any point in the sample can be inferred from the rest:
\[
\mathbb{E}[\text{Coverage}]= \mathbb{E}_{T \sim D_X^{|S'|+1}}\left[\frac{1}{|T|}\#\{x_i \in T : T \setminus \{x_i\} \yields{} x_i\}\right].
\]
Recall that with probability at least $\frac{5}{9}$ we can find and, with probability $1-\delta_r$, correctly order and label a subset of $4k$ points from $S'$. By Observation~\ref{ID:total}, at least $3k$ of these can be inferred from the rest, bounding the right hand side by:
\[
\mathbb{E}_{T \sim D_X^{|S'|+1}}\left[\frac{1}{|T|}\#\{x_i \in T : T \setminus \{x_i\} \yields{} x_i\}\right] \geq  (1-\delta_r)\frac{5}{9}\frac{3k}{32k+16} > \frac{1}{60},
\] 
where we have assumed $\delta_r<1/2$. Then for any constant $c_1>0$ we have:
\[
	\frac{1}{60} < \mathbb{E}[\text{Coverage}] \leq \Pr[\text{Coverage} \geq c_1] + \Pr[\text{Coverage} < c_1]  c_1,
\] which for small enough $c_1$ gives:
\[
	\Pr[\text{Coverage} \geq c_1] > \frac{\frac{1}{60} - c_1}{1-c_1} > c_1
\] 
Accounting for the fact that we have assumed our comparisons and labels are correct, our weak learner has coverage $>c_1$ with probability at least $(1-\delta_r)2c_1> c_1$ for $\delta_r< \frac{1}{2}$.
\\
\paragraph{Query Complexity:} Now, we compute the number of queries made by the weak learner. Let $c_3 = m_2/\log n$ where $m_2$ is the point-wise movement as defined in Observation \ref{order2}. Using the same notation as \cite{braverman2009sorting}, we let (setting $\alpha = \bigo\left (\frac{\log \frac{1}{\delta_r}}{\log n} \right ), A = \np^{-2}$ in constants of \cite{braverman2009sorting}) \[
    c_3 = \bigo \left(\np^{-5} \log \frac{1}{\np} \left ( 1 + \frac{\log \frac{1}{\delta_r}}{\log(n)}\right)\right); \quad c_5 = \bigo \left(c_3+\left(\log\frac{1}{\delta_r}\right)^{\frac{1}{3}}\right);\quad c_6 = \bigo \left(\frac{\log \frac{1}{\delta_r}}{\log n}\right); \quad c_8 = 4(Ac_6 + 6c_3)
\]
Using \cite[Lemmas 31 and 32]{braverman2009sorting}, the number of queries made in the sorting $n$ points (which includes dynamic programming step on $n$ points and slotting $n$ points) and slotting additional $|S'| = 32k+16$ points are 
\begin{align*}
    q_{wl}(\delta_r) &= \underbrace{\bigo(c_5 n \log n)}_{\text{dynamic programming step}} + \underbrace{\bigo(c_8 \log n + 3Ac_6 \log n)}_{\text{slotting a single point}} \cdot (n + 32k + 16)\\ 
    &= \tilde \bigo \left( \frac{k}{\np^{10}} \log \frac{1}{\delta_r}\log \frac{k}{\np\delta_r} \right).
\end{align*} with probability $1-\delta_r$.
Since we do not want our number of queries to be probabilistic, if our learner does not complete after $q_{wl}(\delta_r)$ queries, we stop and output all 0's. This increases our error probability by $\delta_r$.
\paragraph{Time Complexity:} Using an algorithm from \cite[Theorem 30]{braverman2009sorting}, we can sort $n$ points with noisy comparisons in time $n^{c_4}$ where $c_4 = \bigo (\np^{-5} \log \frac{1}{\np}(1+ (\log \frac{1}{\delta_r}) (\frac{1}{\log n})))$ with probability $1-\delta_r$. Since slotting a point in worst case takes $\bigo (n)$ time, we can slot $\bigo(k)$ points in time $\bigo(k n)$. This gives us the total time taken by the weak learner as 
\[T_{wl}(\delta_r) = \bigo(n^{c_4}) + \bigo (kn); \quad \text{where } c_4 = \bigo \left(\np^{-5} \log \frac{1}{\np}\left ( 1 + \frac{\log \frac{1}{\delta_r}}{\log(n)}\right)\right )\]
Therefore, the time complexity of the algorithm is $\poly(k,\frac{1}{\delta_r})^{\tilde{O}\left (\frac{1}{\np^5} \right )}$. Once again taking the strategy of outputting all 0's if the algorithm does not  complete in time $T_{wl}(\delta_r)$, we lose another error factor of $\delta_r$, making the algorithm all together $3\delta_r$-reliable.
\end{proof}
With our weak learner in hand, all that is left for the proof of Theorem~\ref{massart} is Step 3: stringing together copies  of the weak learner through rejection sampling.
\begin{proof}[Proof of Theorem~\ref{massart}]
Let $\delta_r^w$ and $\delta_u^w$ be reliability and usefullness parameters for our weak learner. Recall that Lemma~\ref{lemma:weak-learner-massart} gives a $3\delta^w_r$-reliable weak learner with coverage $c_1$ with probability $c_1$. Applying this weak learner $\bigo (\log(1/\delta^w_u))$ then amplifies this probability to at least $1-\delta^w_u$.
\\
\\
Restricting to the distribution of un-inferred points via rejection sampling, we repeat the above process until our coverage reaches $1-\varepsilon$. Assume each repetition is successful, then after $t$ steps our coverage is:
\[
\text{Coverage} \geq 1-c_1^t.
\]
Setting $t$ to $\bigo (\log(1/\varepsilon))$ is then sufficient to set the right hand side to $1-\varepsilon$. However, each repetition in this process degrades the overall probability of usefulness. In order to get an overall guarantee of $\delta_u$, we must adjust our initial $\delta^w_u$ to:
\begin{align*}
\delta_u^w &= \bigo\left (\frac{\delta_u}{\log \left( \frac{1}{\varepsilon} \right)} \right).
\end{align*}
Similarly, since we apply the weak learner $\bigo(\log(1/\varepsilon) \log(1/\delta^w_u))$ times, we adjust our $\delta_r^w$ to 
\begin{align*}
    \delta_r^w &= \bigo \left (\frac{\delta_r}{\log \left( \frac{1}{\varepsilon} \right)\log\left(\frac{1}{\delta_u^w}\right)} \right).
\end{align*}
\paragraph{Query Complexity:}
In total, we run our weak learner at most $\bigo\left (\log \left( \frac{1}{\varepsilon} \right)\log\left(\frac{1}{\delta^w_u}\right)\right )$ times, giving a query complexity of: 
\begin{align*}
    q(\varepsilon, \delta_r, \delta_u) &= \bigo\left  (\log \left( \frac{1}{\varepsilon} \right)\log\left(\frac{1}{\delta^w_u}\right)\right ) \cdot q_{wl}(\delta_r^w)\\
    &=\tilde \bigo \left(\log\frac{1}{\varepsilon}\log\frac{1}{\delta_u}\right)\cdot \tilde \bigo \left( \frac{k}{\np^{10}} \log \frac{1}{\delta^w_r}\log \frac{k}{\np\delta^w_r} \right)\\
    &= \tilde \bigo \left(\frac{k}{\np^{10}} \log\frac{1}{\varepsilon} \log \frac{1}{\delta_u} \log^2 \frac{1}{\delta_r}\right) .
\end{align*}
\paragraph{Sample Complexity:} At each step of our algorithm, we restrict to the distribution of un-inferred points through rejection sampling. By itself, this poses a problem: what if we have inferred much of the space early and our algorithm continually rejects points? To combat this, we note that we can estimate the measure of remaining un-inferred points by how many samples we have to draw before finding one. Formally, if at any step we draw $2\log(1/\delta_u)/\varepsilon$ inferred points in a row, then by a Chernoff bound the coverage of our learner is $1-\varepsilon$ with probability at least $1-\delta_u$. Let $n$ be the sample size as defined in Lemma~\ref{lemma:weak-learner-massart}. Since our algorithm only queries a total of $N=\bigo \left (n\log \left( \frac{1}{\varepsilon} \right)\log\left(\frac{\log \left( \frac{1}{\varepsilon} \right)}{\delta_u}\right)\right)$ points, the same result holds by a union bound if our algorithm stops after rejecting $2\log(N/\delta_u)/\varepsilon$ points in a row. This means that we can bound the total number of samples drawn by
\[
n(\varepsilon,\delta_r,\delta_u)=O\left(\frac{N\log(N/\delta_u)}{\varepsilon}\right).
\]

\paragraph{Time Complexity:} The time complexity of our algorithm has two main components: the complexity of finding an MLE order in the weak learner, and the complexity of rejection sampling. We already computed the time complexity of the weak learner in Lemma~\ref{lemma:weak-learner-massart} as
$T_{wl}(\delta_r)=\poly(k,\log(\frac{1}{\delta_r}))^{\tilde{O}\left (\frac{1}{\np^5} \right )}$. Since, we run our weak learner at most $\bigo\left (\log \left( \frac{1}{\varepsilon} \right)\log\left(\frac{1}{\delta^w_u}\right)\right )$ times, the time complexity for finding MLE is $\poly(k,\log\frac{1}{\varepsilon},\log\frac{1}{\delta_u},\log(\frac{1}{\delta_r}))^{\tilde{O}\left (\frac{1}{\np^5} \right )}$.
\\
\\
It remains to compute the time complexity of rejection sampling. Recall that the we sample at most $n(\varepsilon,\delta_r,\delta_u)$ points total in our process. For each point, we run an LP in $d+1$ variables with constraints detailed by our previous queries that round. Since the queries our weak learner uses in each round only involve $\tilde{\bigo}(n)$ points, the time complexity of sampling is at most:
\begin{align*}
T_{samp}(\varepsilon,\delta_r,\delta_u) &= n(\varepsilon,\delta_r,\delta_u) \cdot \poly\left(d,k,\frac{1}{\np},\log \frac{1}{\varepsilon},\log \frac{1}{\delta_r},\log \frac{1}{\delta_u}\right)\\
&=\poly\left(d,k,\frac{1}{\np},\frac{1}{\varepsilon},\log \frac{1}{\delta_r},\log \frac{1}{\delta_u}\right).
\end{align*}
Since the total time complexity is order of the sum of sampling and sorting, we get an algorithm that runs in time $\poly(d,k, \frac{1}{\delta_r},\frac{1}{\varepsilon},\log(\frac{1}{\delta_u}))^{\tilde{O}\left (\frac{1}{\np^5} \right )}$.
\end{proof}
\subsection{Average Inference Dimension}
While inference dimension allows us to work over arbitrary continuous distributions, as a complexity parameter it is rather restricting, barring for instance the learning of linear separators in dimensions above two. To generalize to a broader range of classifiers, we will use the framework of average inference dimension introduced in \cite{AID}. In particular, we show that any hypothesis class and distribution with super-exponential average inference dimension may be efficiently learned under Massart noise. As a result, we  provide the first computationally and query efficient learner for non-homogeneous linear separators over s-concave distributions with Massart noise.
\begin{theorem}[Restatement of Theorem~\ref{intro:massart:aid}]
\label{massart:aid}
Consider any hypothesis class $(X,\Hcal)$ and corresponding class of distributions $\mathcal A_{(X,\Hcal),a,f(d)}$. Then, $(X,\Hcal)$ is ARPU-learnable under model $(M(\np), \mathcal A_{(X,\Hcal),a,f(d)})$ in time $\poly(f(d), \frac{1}{\delta_u}, \frac{1}{\varepsilon}, \log(\frac{1}{\delta_r}))^{\tilde{O}\left (\frac{1}{\np^5} \right )}$, uses only $\poly(f(d),\frac{1}{\np},\log(\frac{1}{\varepsilon}),\log(\frac{1}{\delta_r}),\log(\frac{1}{\delta_u})))$ unlabeled samples, and has a query complexity of
\begin{align*}
    q(\varepsilon,\delta_r,\delta_u) = \tilde \bigo \left( \frac{f(d)^{1/a}}{\np^{10}} \log^{2+1/a}\frac{1}{\varepsilon} \log^2 \frac{1}{\delta_r} \log \frac{1}{\delta_u} \right )
\end{align*} for small enough $\delta_r$.
\end{theorem}
Average inference dimension gives a high probability bound on the inference dimension of a finite sample. However, shifting our strategy to directly work with a finite samples introduces a new problem: since our algorithm corrects noise via extra helper points, we may not be able to learn the entire sample. Our first step will be to show that learning most of a finite sample in few queries with high probability is sufficient to learn the entire distribution.
\begin{lemma}\label{sample}
Let $(X,\mathcal H)$ be a hypothesis class, and $D_X$ a distribution over $X$. Let $A$ be an active, inference based learner taking in finite samples $S\sim D^n_X$ with the property that for sufficiently large $n$, $A$ learns a $(1-\varepsilon_1)$ fraction of $S$ with probability $1-\delta$, while querying at most an $\varepsilon_2$ fraction of the points. The expected coverage of $A$ over the entirety of $X$ is at least:
\[
\mathbb{E}[\text{Coverage of} \ A] \geq 1-\delta-\varepsilon_1-\varepsilon_2
\]
\end{lemma}
\begin{proof}
To find the expected coverage of $A$ over the entire distribution $D_X$ based on samples $S$ of size $n$, we look at the probability that an additional randomly drawn point is inferred:
\[
\underset{S \sim D^n}{\mathbb{E}}[\text{Coverage of} \ A] = Pr_{x_1,\ldots,x_{n+1} \sim D^{n+1}_X}[x_1,\ldots x_n \yields{} x_{n+1} ]
\]
We can bound the right hand term by looking at $A$ applied to samples $S'$ of size $n+1$. In particular, if $x_{n+1}$ is learned but not queried by $A$, then because $A$ is an inference based learner, it must be the case that $\{x_1,\ldots x_n\}$ infer $x_{n+1}$. Since the points of $S'$ are drawn i.i.d from $D_X$, the probability that $A$ queries or learns any given point $x_i$ is the same for all $1\leq i \leq n+1$. Because a $1-\varepsilon_1$ fraction of points are learned with probability $1-\delta$ and only an $\varepsilon_2$ fraction of points are queried, the probability that a point is learned but not queried is at least $1-\delta-\varepsilon_1-\varepsilon_2$ by a union bound, which gives the desired bound on $A$'s coverage.
\end{proof}
\begin{proof}[Proof of Theorem~\ref{massart:aid}]
We will argue that the learner presented in Theorem~\ref{massart} satisfies the properties of Lemma~\ref{sample} for a large enough sample size. To prove this, we first examine learning a specific sample with small inference dimension. The coverage over all samples will then follow from the fact that almost all samples have small inference dimension due by Observation~\ref{avg-worst} \cite{AID} and our assumption on average inference dimension.
\\
\\
Because we are considering a fixed sample $S$, the weak learner draws uniformly without replacement from $S$ (denoted $x \sim S$) rather than from the distribution itself. All required symmetry arguments still hold in this regime, as the order that points are pulled is still uniformly random. The expected coverage of our learner over $S$ is thus the same as for $X$ in Lemma~\ref{lemma:weak-learner-massart} adjusted for the fact that we sample without replacement:
\begin{align*}
\mathbb{E}[\text{Coverage}] &\geq \underbrace{\frac{n-\bigo_{\np}(\log(n))}{|S|}}_{\text{Coverage on $x_1,\ldots,x_n$}} + \underbrace{\left (1-\frac{n}{|S|}\right )\Pr\limits_{x_1,\ldots,x_{n+1} \sim S}\left[ \{x_1,\ldots,x_{n}\} \yields{} x_{n+1}\right]}_{\text{Coverage on rest of sample}}
\end{align*}
and hence
\begin{align*}
\mathbb{E}[\text{Coverage}] &\geq \Pr\limits_{x_1,\ldots,x_{n+1} \sim S}\left[ \{x_1,\ldots,x_{n}\} \yields{} x_{n+1}\right] -\frac{\bigo_{\np}(\log(n))}{|S|}\\
\end{align*}
Assume for now that $|S|$ is large enough that the subtracted term is negligible. To analyze the remaining coverage probability, assume that $n$ satisfies the constraints of Lemma~\ref{lemma:weak-learner-massart} with $k=\tilde{\Theta}(f(d)^{1/a}\log^{1/a}(|S|))$, and further that $S$ has inference dimension $k$. Then by the arguments in Lemma~\ref{lemma:weak-learner-massart}, this probability over the sample itself and noisy oracles is constant. Further, as long as $S$ is sufficiently large, we can get coverage $1-\varepsilon$ with probability $1-\varepsilon$ by applying the same argument restricted to the subset of un-inferred points $\bigo\left(\log^2 \left( \frac{1}{\varepsilon} \right)\right)$ times. This argument only fails when there are no longer $n$ remaining points for our weak learner to use, but as long as $|S|=\omega(\frac{n}{\varepsilon})$, this will not affect our coverage. Since Lemma~\ref{sample} also only allows the learner to query a $\varepsilon$ fraction of points, we set $S$ to:
\begin{align*}
|S| &= \Theta\left (\frac{n\log^2 \left( \frac{1}{\varepsilon} \right)}{\varepsilon}\right )\\
n &= \tilde{\Theta} \left (\frac{f(d)^{1/a}}{\np^5}\log^{1/a}(|S|)\log\left(\frac{1}{ \delta_r}\right) \right ),
\end{align*}
which also validates our assumption that $\frac{\bigo_{\lambda}(\log(n))}{|S|}$ is negligible (we lose less than $\varepsilon$ over all iterations). To apply Lemma~\ref{sample}, it is sufficient to have a learner $A$ such that:
\[
\Pr\limits_S[\text{Coverage of} \ A > 1-2\varepsilon] > 1-2\varepsilon.
\]
Because $|S|>\Omega\left (\frac{1}{\varepsilon} \right )$, S has inference dimension $k$ with probability at least $1-\varepsilon$ by Observation \ref{avg-worst} \cite{AID}. Combining this with the fact that our algorithm has a $1-\varepsilon$ probability of achieving $1-2\varepsilon$ coverage when the inference dimension is $k$ proves this claim.
\\
\\
Finally, by Lemma~\ref{sample}, our learner has expected coverage is $\geq 1-5\varepsilon$ over the entire space. To get the desired coverage probability, we run the algorithm over $\bigo(\log(1/\delta_u))$ samples, setting $\delta_r$ to $\delta_r/\log(1/\delta_u)$ to amend the degradation of correctness over repetition. Then by the same argument as Theorem~\ref{massart}, our query complexity is:
\[
    q(\varepsilon,\delta_r,\delta_u) = \tilde \bigo \left( f(d)^{1/a}\log^{1/a}|S|\frac{1}{\np^{10}} \log^2\frac{1}{\varepsilon} \log^2 \frac{1}{\delta_r} \log \frac{1}{\delta_u}\right).
\]
Sample and time complexity follow similarly to Theorem~\ref{massart}.
\end{proof}

\section{Generalized Tsybakov Noise Condition}\label{Sec:GTNC}
The Massart noise model does well to capture situations with adversarial bounded noise, but even in a realistic non-adversarial scenario, error may not be bounded away from $1/2$. One might think, for instance, that label noise should be bounded as a function of the distance to the Bayes optimal classifier, reaching purely random labels on the decision boundary itself. Likewise, comparisons between arbitrarily close points should be difficult, with error approaching $1/2$ as well. This motivates us to study the Tsybakov Low Noise condition, a popular instantiation of distance-based noise. However, learning in this unbounded regime is harder, as evidenced by polynomial query lower bounds \cite{wang2016noise,xu2017noise}, and the lack of computationally efficient algorithms for the model. In order to ARPU-learn in this regime, we need to introduce more stringent restrictions than for Massart noise. First, instead of allowing any set system with finite inference dimension, we will only consider non-homogeneous linear separators. Second, we will either assume some margin $\gamma$, or that the distribution satisfies certain weak concentration and anti-concentration bounds, a property which implies our earlier assumption for Massart noise of super-exponential average inference dimension.
\subsection{Finite Inference Dimension and Margin}
In this section, we will consider ARPU-learning hyperplanes over any continuous distribution with finite inference dimension and margin. Note that in the GTNC model, introducing margin bounds the error on label queries away from $1/2$. Thus our results should informally be viewed as saying the following: comparison queries with \textit{unbounded error} exponentially improve query complexity over label queries with \textit{bounded error} in the ARPU-learning model. Indeed, although we have picked a specific model of bounded label error in this case, trading for another model such as Massart noise on labels causes no significant change to our upper or lower bound.
\\
\\
As in the case of Massart noise, we will first show the gap in query complexity between label only and comparison ARPU-learning. Our previous method showed an infinite gap between the two regimes, but the assumption of a non-zero margin requires a different argument. In this case, we will show a family of examples in which comparisons provide an exponential improvement.
\begin{lemma}[Restatement of Lemma \ref{intro:TNC:margin:lb}]
Let $X \in \mathbb{R}^d$ be the $d$-dimensional hypercube $\{0,1\}^d$ modified to have a ball of radius $\frac{1}{4\sqrt{d}}$ centered about each point. The query complexity of ARPU-learning $(X,H_{d,\frac{1}{4\sqrt{d}}})$ under model $(\gtnc(g_L,g_U,\frac{1}{4\sqrt{d}}), \mathcal C_X)$ is at least:
\[
q(1/4,1/8,1/16) \geq 2^{d-1}
\]
\end{lemma}
\begin{proof}
For simplicity, the adversary will pick the uniform distribution from $\mathcal C_X$, and the noiseless case from $(\gtnc(g_L,g_U,\frac{1}{4\sqrt{d}}), \mathcal C_X)$. Further, by Yao's minimax principle it is sufficient to show there is a distribution over hyperplanes in $H_{d,\frac{1}{4\sqrt{d}}}$ for which no learner can achieve at least $3/4$ coverage with perfect correctness with greater than $3/4$ probability. Let the adversary pick the uniform distribution over the $2^d$ hyperplanes which truncate corners of the hypercube, e.g.
\[
\sum\limits_{i=1}^d x_i = 1/2.
\]
Note that these hyperplanes have margin $\frac{1}{4\sqrt{d}}$, so they lie in $H_{d,\frac{1}{4\sqrt{d}}}$, and that each one may be seen as selecting a single ball to be negative. Given any set strategy, the learner can only query points in $2^{d-1}$ out of $2^d$ balls. The probability that one of the balls the learner queries is the negative ball is at most $1/2$. If the learner does not locate the negative ball, to have coverage $3/4$ it must label half of the remaining space with no additional queries. However, any set strategy from the learner in this case will have an incorrect label with probability at least $1/2$ since the negative ball is uniformly distributed over the remaining balls. Thus any learner that has $3/4$ coverage with probability more than $3/4$ must incorrectly label some point, violating the conditions of ARPU-learning.
\end{proof}
By an argument based on minimal-ratio (margin normalized by the maximum function value) from \cite{kane2017active}, the inference dimension of the above hypothesis class is $\tilde{\bigo}(d)$. We will prove that this implies a comparison based algorithm that only makes $\poly(d)$ queries.
\begin{algorithm}[t!]
\SetAlgoLined
\KwResult{Returns an $\delta_u$-useful, $\delta_r$-reliable learner with coverage $1-\varepsilon$ for model ($\gtnc(g_L,g_U, \varepsilon_0), \mathcal{C}_X$)}
\nonl \textbf{Input:} Noisy oracles $Q_L, Q_C \in \gtnc(g_L,g_U, \varepsilon_0)$, unknown distribution $D_X \in \mathcal{C}_X$\\
\nonl \textbf{Parameters:} 
\begin{itemize}
    \item Inference dimension $k$, input dimension $d$, and margin $\gamma$
    \item Sample sizes $n=\poly(k,d, \frac{1}{\varepsilon_0},\log(\frac{1}{\delta_r}),\frac{1}{\gamma})$, $m_{c}=d \log(d+1) n$, and $m_s = \poly(k)$
    \item Iteration cap $T = \poly(\log \frac{1}{\delta_r}, \log \frac{1}{\delta_u}, \log \frac{1}{\varepsilon}, k, \frac{1}{\gamma})$
    \item Sample cap $C = \poly(k,d, \frac{1}{\gamma}, \log \frac{1}{\delta_u}, \log \frac{1}{\delta_r}, \frac{1}{\varepsilon_0}, \frac{1}{\varepsilon})$
    \item Equitability constants $\varepsilon_T$ and $\gamma'$ (Equation \eqref{eqn:equitable-constants})
    \item Set of linear program constraints $LP=\{\}$
\end{itemize}
\nonl \textbf{Algorithm:}
\begin{enumerate}[leftmargin=*]
    \item Sample $S \sim D_X^n$ restricted to points un-inferred by $LP$. \item Test $S$ for noise by checking for $\varepsilon_T$-equitable subsets of size $2c+m$.
    \item If $S$ measures as noisy i.e. at least one $\varepsilon_T$-equitable subset $S_{eq}$ is found:\begin{enumerate} 
        \item Sample $S'\sim D^{m_c}_X$ restricted to points un-inferred by $LP$. 
        \item Update $LP$ constraints using comparisons and labels of all $x\in S'$ for which $S_{eq} \cup x$ is $\frac{g_L(\gamma')}{2}$-equitable.
    \end{enumerate} 
    Else S measures as having only a small amount of noise i.e. no $\varepsilon_T$-equitable subset was found:
    \begin{enumerate}
        \item Sort $S\cup \zo$ into the MLE order via noisy oracles $Q_L$ and $Q_C$.
        \item Sample $S' \sim D^{m_s}_X$ restricted to points un-inferred by $LP$ and insert into the order of $S$.
        \item Update $LP$ constraints using comparisons and labels of points in $S'$ separated by $\tilde\Omega(n^{3/4})$ from each other and from $\zo$.
    \end{enumerate}
    \item If at any point $C$ inferred samples are drawn in a row, return the current $LP$. Repeat from step one until iteration cap $T$ is reached and return $LP$.
\end{enumerate}
 \caption{ARPU-learning with Finite Inference dimension and Margin under Generalized Tsybakov Low Noise Condition}
 \label{algo:gtnc-algo}
\end{algorithm}
\begin{theorem}[Restatement of Theorem~\ref{intro:TNC}]
\label{TNC}
Let $X\subseteq \mathbb{R}^d$ and $(X,H_{d,\gamma})$ have inference dimension $k$ with respect to comparison queries. Then, $(X,H_{d,\gamma})$ is ARPU-learnable under model $(\gtnc(g_L,g_U,\varepsilon_0), \mathcal C_X)$ with query complexity:
\begin{align*}
q(\varepsilon,\delta_r,\delta_u) &=\tilde{\bigo} \left (\frac{k^{10}}{\left(g_L\circ G_8 \circ \frac{G_4(\gamma')}{2}\right)^{14}}d\log^{2}\left( \frac{1}{\delta_r} \right)\log\left(\frac{1}{\varepsilon}\right)\log\left(\frac{1}{\delta_u}\right)\right ).\\
\end{align*}
Where
\begin{align*}
\gamma' = \min\left(\frac{\gamma}{2d},\frac{\varepsilon_0}{2}\right), G_c(x) = g_U^{-1}\left( \frac{g_L(x)}{c}\right)
\end{align*}
\end{theorem}
See Algorithm~\ref{algo:gtnc-algo}. Unlike the Massart case, we can no longer directly rely on the sorting algorithm of \cite{braverman2009sorting}, as the point-wise movement guarantees rely on bounded noise. Instead, we rely on the fact that we can, with high probability, check the level of noise of a drawn sample. If the sample is not too noisy, we can modify the bounds of \cite{braverman2009sorting} and apply the same technique. On the other hand, if the sample is very noisy, we use this to infer structural information about the sample and thus learn some fraction of the instance space. Informally, our algorithm follows a similar three step process to the Massart case:
\paragraph{Step 1:} Draw a sample $S\sim D^n_X$, and test $S$ for noise.
\paragraph{Step 2a (high noise):} If $S$ measures as noisy, we identity a subset $S'\subset S$ of points which are close with respect to the underlying hypothesis. Using additional randomly drawn points, we create an inference LP based on the structure of $S'$ to learn a fraction of the instance space.
\paragraph{Step 2b (low noise):} If $S$ measures as having only a small amount of noise, sort $S$ into an MLE order, and apply the same learning strategy as for Massart. 
\paragraph{Step 3:} Restrict $D$ (by rejection sampling) to points un-inferred by the LP in step 2a/b, and repeat steps 1 and 2a/b until coverage has reached $1-\varepsilon$.
\\
\\
At the core of this technique is the ability to detect subsets with high levels of noise, and to certify that they are highly structured. With this in mind, we show that if comparisons on a subset of $S$ look sufficiently random, then almost all points in this subset are clustered together in function value. Formally, we define a cluster as:
\begin{definition}[Cluster]
Let $(X,\Hcal)$ be a set system. Given $h\in \Hcal$ and a sample $S \subseteq X$, $S$ is an $\varepsilon$-cluster with respect to $h$ if
\[
\forall x,x' \in S: |h(x)-h(x')| \leq \varepsilon.
\]
We will often omit ``with respect to $h$'' when $h$ is the function underlying the Bayes optimal classifier.
\end{definition}
We will detect clusters by a measure of randomness we term equitibility, the condition that every element is bigger than about half of the other elements.
\begin{definition}[Equitability]
Let $S$ be a set with comparisons denoted by $\widetilde{<}$ on each pair of elements. For an element $x \in S$, let $v(x)$ denote the number of elements $y \in S$ such that $y \widetilde{<} x$. We call $S$ $\varepsilon$-\textit{equitable} if 
\[
\forall x \in S, \left(\frac{1}{2}-\varepsilon \right)|S| \leq v(x) \leq \left(\frac{1}{2}+\varepsilon \right)|S|
\]
\end{definition}
Because $v(x)$ counts the number of elements less than $x$, it is useful to introduce a new probability parameter:
\[
    \eta_C(x_1,x_2) = \Pr[x_1 \widetilde{<} x_2]
\] i.e. the probability that $x_1$ measures less than $x_2$. Note that $\eta_C(x_1,x_2)$ is either $\beta_C(x_1,x_2)$ or $1-\beta_C(x_1,x_2)$.
\\
\\
In order to distinguish between steps 2a and 2b, we show that if a cluster exists, then with high probability there is a large equitable
subset, and that vice versa, a large equitable subset implies the existence of a large
cluster with high probability. Consider testing a sample $S'$ of size $2c+m$ for equitability. Call the order on $S'$ induced by the underlying classifier the ``true order.'' To start, we examine a single such sample $S'$ and show that with high probability: 
\begin{enumerate}
    \item If $S'$ is a cluster, then it is equitable
    \item If the middle $m$ elements of $S'$ with respect to the true order is not a cluster, then $S'$ is not equitable.
\end{enumerate}
\begin{lemma}
\label{test}
Consider a set $S'$ of size $2c+m$, where $C$ denotes the middle $m$ elements of $S'$ with respect to the true order. Then for $\varepsilon \leq g_L(\varepsilon_0)$, $S'$ satisfies the following properties:
\begin{enumerate}
\item If $S'$ is a $g_U^{-1}(\varepsilon/2)$-cluster, then $S'$ is $\varepsilon$-equitable with probability $1-e^{O(-\varepsilon^2 |S'|)}$.
\item If $C$ is not a $2g_L^{-1}(\varepsilon)$-cluster, then $S'$ is not $(\varepsilon/4)$-equitable with probability at least $1-e^{O(-\varepsilon^2c|S'|)}$.
\end{enumerate}
\end{lemma}
\begin{proof}
Proof of (1). For simplicity, let $n=2c+m-1$ and assume that $x_0,\ldots,x_n$ is the true order of $S'$. Recall that $v(x_j)=v(j)$ is the number of elements that measure as less than $x_j$ and that 
\[
    \eta_C(x_i,x_j) = \Pr[x_i \widetilde{<} x_j]
\]
is the probability that $x_i$ measures less than $x_j$.  We can view $v(j)$ as a random variable given by 
\[
v(j) = \sum\limits_{i\neq j}^{n} \id{x_i \widetilde{<} x_j} = \sum\limits_{i\neq j}^{n}  Bern(\eta_C(x_i,x_j)).
\]
Thus $v(j)$ is a Poisson binomial distribution with parameters $\eta_C(x_i,x_j)$. Let $h^\star$ be the bayes optimal classifier. By assumption, we have for all pairs that $|h^\star(x_i)-h^\star(x_j)| \leq g_U^{-1}(\varepsilon/2)$, and that $\varepsilon \leq g_L(\varepsilon_0)$. Combining these gives
\begin{align*}
g^{-1}_U(\varepsilon/2) \leq g^{-1}_U(g_L(\varepsilon_0)/2) &\leq \varepsilon_0\\
\implies \forall i,j: |h^\star(x_i)-h^\star(x_j)| &\leq \varepsilon_0
\end{align*}
Thus we are in position to apply the upper bound from the GTNC condition (Equation \eqref{eqn:gtnc-1}), which gives for all pairs:
\begin{align*}
\frac{1}{2} \leq \beta_C(x_i,x_j) \leq \frac{1}{2} + \varepsilon/2,  \\
\frac{1}{2} - \varepsilon/2 \leq 1 - \beta_C(x_i,x_j) \leq \frac{1}{2}.
\end{align*} Since, $\eta_C(x_i,x_j)$ is either $\beta_C(x_i,x_j)$ and $1-\beta_C(x_i,x_j)$, we have \[
    \frac{1}{2} - \varepsilon/2 \leq \eta_C(x_i,x_j) \leq \frac{1}{2} + \varepsilon/2
\] This allows us to upper and lower bound the distribution by the binomial distributions $X^u=Bin(n,\frac{1}{2}+\varepsilon/2)$, and $X_l=Bin(n,\frac{1}{2}-\varepsilon/2)$. In particular, for all valid $\eta_C$ and values $j$, we have
\[
X_l \leq v(j) \leq X^u
\]
Where random variables $X,X'$ satisfy $X\leq X'$ if $\forall i$, 
\[
Pr[X \geq i] \leq Pr[X' \geq i]
\]
This means that concentration lower bounds on $X_l$ and upper bounds on $X^u$ transfer to $v(j)$. Now we apply Chernoff bounds to $X_l$ and $X^u$:
\begin{align*}
    Pr \left[v(j) > n \left(\frac{1}{2} + \varepsilon \right) \right] \leq Pr\left[X^u > n \left(\frac{1}{2} + \varepsilon \right) \right] \leq e^{-n\frac{(\varepsilon/2)^2}{1+3\varepsilon/2}}\\
    Pr\left[v(j) < n\left(\frac{1}{2} - \varepsilon \right) \right] \leq Pr \left[X_l < n \left(\frac{1}{2} - \varepsilon \right) \right] \leq e^{-n\frac{(\varepsilon/2)^2}{1-\varepsilon}}
\end{align*}
Union bounding over all values of $j$, the probability that there exists a coordinate outside these ranges is bounded by:
\[
Pr \left [ \exists j: \left|v(j)-\frac{n}{2} \right| \geq n\varepsilon \right] \leq 2(n+1)e^{-n\frac{(\varepsilon/2)^2}{1+3\varepsilon/2}}
\]
Thus our test satisfies the first condition.
\\
\\
Proof of (2). Assume the middle $m=[i,j]$ points of $S$ are not a $2g_L^{-1}(\varepsilon)$-cluster. Since our set is ordered, this implies $h^\star(j)-h^\star(i) > 2g_L^{-1}(\varepsilon)$, and further that the middle point must be at least $g_L^{-1}(\varepsilon)$ far from either $i$, or $j$. Since our argument will be symmetric, assume this to be $i$ without loss of generality. Our strategy will be to bound the random variable
\begin{align*}
V(c) &= \sum_{k=1}^c v(k)\\
&= \sum_{k=1}^c \sum_{l\neq k} \id{x_l \widetilde{<} x_k},
\end{align*}
and use an averaging argument to show that there exists a value $1 \leq x \leq c$ s.t. $v(x)<|S'|(1/2-\varepsilon/4)$

We can decompose $V(c)$ into \[V(c) = \sum_{k=1}^c \sum_{\substack{l< c \\ l \neq k}} \id{x_l \widetilde{<} x_k} + \sum_{k=1}^c \sum_{l> c} \id{x_l \widetilde{<} x_k}, 
\] where the first term is always $\binom{c}{2}$. Because each point left of $i$ is at least $g_L^{-1}(\varepsilon) \leq \varepsilon_0$ far away from the right half of $|S'|$, we can bound the second term as for any $v$, \begin{align*}
    \Pr\left[\left (\sum_{k=1}^c \sum_{l> c} \id{x_l \widetilde{<} x_k}\right ) > v\right] &\geq \Pr\left[\sum_{k=1}^c\underbrace{\bin(1/2,m/2)}_{\text{Points up to } |S|/2} + \sum_{k=1}^c\underbrace{\bin(1/2-\varepsilon,c+m/2)}_{\text{Points after } |S|/2} > v\right]\\
    &= \Pr\left[ \bin \left(1/2, \frac{cm}{2} \right)+\bin \left(\frac{1}{2}-\varepsilon,c^2+\frac{cm}{2} \right) > v\right]\\
\end{align*}
A Chernoff bound gives
\[
\Pr\left[V(c) > c|S'|(1/2-\varepsilon/4)\right] \leq e^{-\frac{\varepsilon^2c(c+m)}{24}}.
\]
Then an averaging argument shows that \[
\Pr\left[\exists x ~\text{s.t}~ v(x) < |S'|(1/2-\varepsilon/4)\right] \geq \Pr\left[V(c) < c|S'|(1/2-\varepsilon/4)\right] > 1 - e^{-\frac{\varepsilon^2c(c+m)}{24}}
\]
\end{proof}
We are not quite done with our cluster detection algorithm, as our goal will be to test for clusters sublinear in the size of our main sample $S$. Lemma~\ref{test} is enough to show that if such a cluster exists some subset will measure as equitable, but we need to prove that any equitable subset of $S$ contains a cluster. For large enough $c$, this is true with high probability.
\begin{corollary}
\label{cor:equitable}
Let $S$ be a sample of size $n$, and $\varepsilon \leq \frac{g_L(\varepsilon_0)}{4}$. For all subsets $S' \subseteq S$ of size $|S'|=(2c+m)$ satisfying:
\begin{align*}
c &\geq \frac{48\log(n)+\log(1/\delta)}{\varepsilon^2}
\end{align*}
the following guarantees hold:
\begin{enumerate}
    \item If $S$ contains a $g_U^{-1}(\varepsilon/2)$-cluster of size $2c+m$, then at least one $S'$ is $\varepsilon$-equitable with probability at least $1-\delta$.
    \item For all $\varepsilon$-equitable $S'$, $C$, the middle $m$ elements of $S'$ with respect to the true order, is a $2g_L^{-1}(4\varepsilon)$-cluster with probability at least $1-\delta$.
\end{enumerate}
\end{corollary}
\begin{proof}
Both statements follow from applying Lemma~\ref{test} to subsets $S' \subset S$ of size $2c+m$.
\\
\\
Proof of (1). By assumption, $S$ contains at least one subset $S'$ of size $2c+m$ which is a cluster. Applying statement (1) of Lemma~\ref{test} to $S'$ gives that $S'$ is equitable with probability at least $1-\delta$.
\\
\\
Proof of (2). We prove statement (2) by the contrapositive: with probability $1-\delta$, all subsets $S'$ such that $C$ is not a $2g_L^{-1}(4\varepsilon)$-cluster are not equitable. This follows from statement (2) of Lemma~\ref{test} and union bounding over all ${\binom{n}{|S'|}}$ possible subsets.
\end{proof}
We can now explain step 1 of our algorithm, cluster detection, in a bit more detail.
\paragraph{Step 1:} Draw a sample $S \sim D^n_X$, and set $c$ and $m$ corresponding to the desired cluster sizes for testing. For every subset $S' \subset S$ of size $2c+m$, check whether $S'$ is $\varepsilon$-equitable. By the contrapositive of Corollary~\ref{cor:equitable} (1), if no such $S'$ is $\varepsilon$-equitable, then no $g_U^{-1}(\varepsilon/2)$-cluster exists in $S$. Similarly, by Corollary~\ref{cor:equitable} (2) if $S'$ is $\varepsilon$-equitable, then it contains a $2g_L^{-1}(4\varepsilon)$-cluster $C$ of size $m$.
\\
\\
With step 1 out of the way, we will prove that steps 2a and 2b provide reliable learners with good coverage as long as the cluster assumption from step 1 holds. Since our focus has been on clusters, we will begin by showing how to build the learner for step 2a. Recall that to apply the symmetry argument of \cite{kane2017active} for Massart noise, we had to slot a set of extra points. We will adhere to a similar strategy for step 2a in which we slot an extra set of points and find a cluster there rather than in $S$ itself. To find this cluster, our first goal will be to prove that additionally drawn points measure as equitable with $S'$ if and only if they are in the same cluster as $C$.
\begin{lemma}
\label{eq:slotting}
Let $S$ be a $\varepsilon$-equitable set of size $m+2c$ satisfying the conditions of Corollary~\ref{cor:equitable}. Let $C$ be the subset of $S$ which is the $2g_L^{-1}(4\varepsilon)$-cluster specified in Corollary~\ref{cor:equitable}, and let $S'$ be a set of independently drawn points.  Further, choose $\varepsilon,m$ to satisfy:
\begin{align*}
    \varepsilon &\leq  \left ( \frac{g_L\left (\frac{g_U^{-1}\left ( \frac{g_L\left(\frac{\varepsilon_0}{2}\right)}{4}\right )}{2}\right )}{4}\right )\\
    m &\geq \frac{9\log(2|S'|/\delta)}{\np_1^2}.
\end{align*}
The following guarantees hold $\forall x \in S'$ with probability at least $1-\delta$.
\begin{enumerate}
    \item If $C \cup \{x\}$ is a $2g_L^{-1}(4\varepsilon)$-cluster, then $S \cup \{x\}$ is $\np_1=2g_U(2g_L^{-1}(4\varepsilon))$-equitable.
    \item If $S \cup \{x\}$ is $\np_1$-equitable, then $C \cup \{x\}$ is a $g_L^{-1}(2\np_1)+2g_L^{-1}(4\varepsilon)$-cluster.
\end{enumerate}
\end{lemma}
\begin{proof}
Proof of (1). Assume that $C \cup \{x\}$ is a $2g_L^{-1}(4\varepsilon)$-cluster. Note that since our assumption on $\varepsilon$ implies $\varepsilon \leq \frac{g_L(\varepsilon_0/2)}{4}$, we have:
\[
2g_L^{-1}(4\varepsilon) \leq \varepsilon_0.
\]
Then GTNC allows us to bound $\eta_C(x,y)$ for all $y \in C$:
\[
|\eta_C(x,y) - 1/2| \leq g_U(2g_L^{-1}(4\varepsilon))=\np_1/2.
\]
To show that $v(x) \leq |S|(1/2 + \np_1)$, we assume the worst case -- that all elements of $S$ are smaller than $x$. Since $C \cup \{x\}$ is a cluster, we can bound $v(x)$ by the following Binomial:
\[
v(x) \leq Bin(1/2+\np_1/2,m+c)+c.
\]
The probability that $v(x) > |S|(1/2+\np_1)$ is then given by a Chernoff bound as
\[
Pr[v(x) > |S|(1/2+\np_1)] \leq e^{-\frac{\np_1^2m}{9}}.
\]
We can bound the probability that $v(x) < |S|(1/2-\np_1)$ by rehashing the same argument for $|S|-v(x)$, the number of elements $x$ is less than. Thus the probability that $C \cup \{x\}$ is not $\np_1$-equitable is
\[
Pr[C \cup \{x\} \ \text{is not} \ \np_1 \text{-equitable}] \leq 2e^{-\frac{\np_1^2m}{9}}.
\]
Union bounding over $S'$ completes the proof.
\\
\\
Proof of (2). Similar to the proof of statement (2) of Corollary~\ref{cor:equitable}, we prove the contrapositive: that all $C \cup \{x\}$ which are not $g_L^{-1}(2\np_1)+2g_L^{-1}(4\varepsilon)$-clusters are not $\lambda_1$-equitable with high probability. Assume $C \cup \{x\}$ is not a $g_L^{-1}(2\np_1)+2g_L^{-1}(4\varepsilon)$-cluster. Since $C$ is a $2g_L^{-1}(4\varepsilon)$-cluster, $\forall y \in C$ we have 
\[
|h^\star(x)-h^\star(y)|> g_L^{-1}(2\np_1)
\]
Since we have assumed $g_L^{-1}(2\np_1) < \varepsilon_0$, GTNC gives $\forall y \in C$:
\[
\left | \eta_C(x,y) - 1/2 \right | > 2\np_1.
\]
Since $C \cup \{x\}$ is not a cluster, it must either be the case that $\forall y \in C, x>y$, or $\forall y \in C, x<y$. Assume the latter without loss of generality. We can bound $v(x)$ by a Binomial:
\[
v(x) \leq Bin(1/2-2\np_1,m+c)+c.
\]
A Chernoff bound then gives:
\[
Pr[v(x) > |S|(1/2-\np_1)] \leq e^{-3\np_1^2m}.
\]
Union bounding over $S'$ proves the contrapositive, completing the proof.
\end{proof}
Knowing that additionally drawn points which measure as equitable with $S$ come from a cluster, we can feed them into an inference LP based on this assumption. However, to infer remaining points in the instance space, the LP must also know the label of the cluster we feed in. Since we are assuming our points have some margin $\gamma$, we can solve for the label of the cluster with high probability by majority vote.
\begin{lemma}[Cluster Labeling]
\label{cluster:label}
Assume that a set $S$, $|S|\geq \frac{2\log(1/\delta)}{g_L(\gamma)^2}$, consists entirely of one label and has margin $\gamma$ with respect to the decision boundary. The probability that this true label differs from the majority label measured by the oracle $Q_L$ is at most $\delta$.
\end{lemma}
\begin{proof}
This follows from applying a Chernoff bound to the fact that each point has at least a $1/2 + g_L(\gamma)$ probability of being correct.
\end{proof}
Finally, we need to show that the LP based upon the structure and label of clustered points has good coverage. We do this by an argument inspired by inference dimension: that given a $\gamma/d$-cluster $C$ of large enough size, there exists a point in $x \in C$ such that the knowledge that $C-\{x\}$ is a cluster is sufficient to infer the label of $x$. This will allow us to use the symmetry argument of \cite{kane2017active} to show that step 2a has good coverage.
\begin{lemma}
\label{cluster:id}
Let $X$ be a set, and $H_{d,\gamma}$ the set of hyperplanes with margin $\gamma$ with respect to $X.$ Consider a query set $Q$ containing a \textit{cluster query} along with the standard label queries. Given a subset $S$, a cluster-query returns 1 if $S$ is a $\gamma/d$-cluster, and $0$ otherwise. Then for any $\gamma/d$-cluster $C \subseteq X$ of size at least $24d\log(d+1)$:
\[
\forall h \in H_{d,\gamma}, \exists x \in C \ \text{s.t.} \ Q(C \setminus \{x\}) \yields{h} x
\]
\end{lemma}
\begin{proof}
A $\gamma/d$-cluster $C=\{x_1,\ldots,x_n\}$ infers a point y if there is a solution to the following system of linear equations:
\begin{align}
    \sum a_i &= 1\nonumber\\
    \sum a_i x_i &= y \label{eq:y}\\
    \sum |a_i| &\leq d + 1 \label{eq:d}
\end{align}
Informally, because $C$ is a $\gamma/d$-cluster and all points have margin $\gamma$, it infers the labels not just of points in its convex hull, but in a $d$ times expansion of the hull. We will show that a large enough cluster $C$ must contain some point $y$ s.t. $C \setminus \{y\}$ infers $y$. Our strategy relies on the fact that if $C$ does not infer $y$, adding $y$ to $C$ expands the volume of its convex hull by a multiplicative factor. Since we can upper bound the volume of the convex hull of $C$ by the volume of the largest simplex times the size of a decomposition of $C$ into simplices (a triangulation), this multiplicative volume expansion contradicts the upper bound for large enough $C$.
\\
\\
In order to prove that adding a point multiplicatively expands the volume of the convex hull, we will need to prove the existence of a certain affine linear function. In particular, if $C$ and $y$ are such that this system of equations has no solution, then there exists an affine function $L$ such that:
\begin{align}\label{claim1}
    L(y) - L(x_{\max}) > \frac{d}{2} (L(x_{\max}) - L(x_{\min})) \geq 0,
\end{align}
where $x_{max}$ is the $\underset{i}{\text{argmax}} \ L(x_i)$, and $x_{min}$ is the corresponding argmin.
\\
\\
\begin{figure}\centering
	\includegraphics[width=0.5\linewidth]{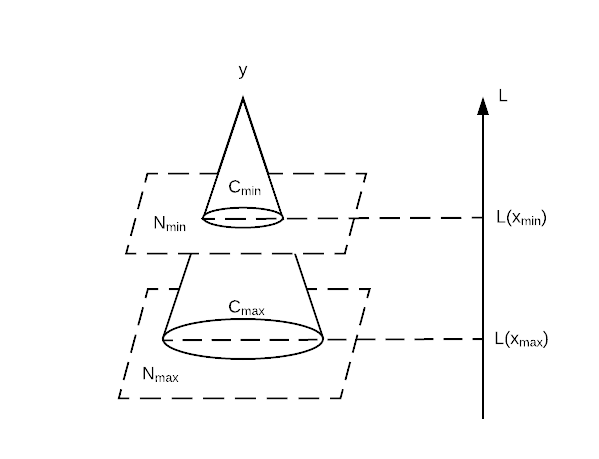}
	\caption{The above image illustrates the construction of sets $C_{\min}$ and $C_{\max}$ which sandwich the cluster $C$.}
	\label{fig:cluster}
\end{figure}
\textit{Proof of \eqref{claim1}:} Since we have assumed the system has no solution, there must be a positive real linear combination of the inequalities and real linear combination of the equalities that sum to the contradiction $1 \leq 0$ by LP-duality. Since $a_i,x_i,$ and $y$ do not appear in this contradiction, the linear combinations of Equation~\eqref{eq:y} and \eqref{eq:d} must cancel. To see this explicitly, let the linear combination of \ref{eq:y} be denoted $T$, then the equality becomes:
\[
\sum a_i T(x_i) = T(y).
\]
Note that Equation \eqref{eq:d} in a truly linear form is a set of $2^d$ equations $\sum a_ie_i$ for $e \in \{-1,1\}^d$. The positive real linear combination of these terms is then of the form
\[
\sum a_i b_i,
\]
for some $b \in \mathbb{R}^d$.  Since these two sums must cancel, we get that the $b_i$ are in fact $-T(x_i)$. Summing the two equations then gives:
\begin{align*}
&\sum a_i T(x_i) + \sum a_i b_i = 0 \leq T(y) + (d+1)\max\limits_i|T(x_i)|.
\end{align*}
Now define $L=-T$, which remains an affine linear function. This sign only affects the left-hand term, and thus we get:
\begin{align*}
L(y) &\geq (d+1)\max\limits_i|L(x_i)|\\
\implies L(y) - L(x_{\max}) &\geq (d+1)\max\limits_i|L(x_i)| - L(x_{\max})\\
\implies L(y) - L(x_{\max}) &\geq d \max\limits_i|L(x_i)|.
\end{align*}
Noting that $L(x_{max})-L(x_{min})$ is at most $2\max\limits_i|L(x_i)|$ proves the claim.
\\
\\
Using the function $L$ we can show how $C$ expands in volume when adding an un-inferred point:
\begin{align}\label{claim2}
\frac{\vol \ConvHull(C,y)}{\vol \ConvHull(C)} \geq \frac{e^2}{e^2-1}.
\end{align}
\textit{Proof of \eqref{claim2}:} Our strategy will be to sandwich the convex hull of $C$ in the difference of two cones defined by $L$ with apex $y$. For arbitrary points $x \in C$, let $h(x,y)$ be the line passing through $x$ and $y$, $N(x_{\min})$ be the plane given by $L(x) = L(x_{\min})$ and $N(x_{\max})$ be the plane given by $L(x) = L(x_{\max})$. The cone which does not contain $C$ is then defined by its apex $y$ and base $C_{max}$:
\[
C_{max} = \{x' : x'\in h(x,y) \cap N(x_{\max}) \text{ for } x \in \ConvHull(C)\}
\]
Likewise, we define the cone that contains both Cone($C_{max},y$) and ConvHull($C$) as the cone with apex y and base $C_{min}$:
\[
C_{min} = \{x' : x'\in h(x,y) \cap N(x_{\min}) \text{ for } x \in \ConvHull(C) \}.
\]
We refer to these cones respectively as Cone($C_{max},y$) and Cone($C_{min},y$). Note that $C_{max}$ is similar to $C_{min}$ and that Equation~\eqref{claim1} bounds the ratio in volume of these cones:
\[
\frac{\vol \cone(C_{min},y)}{\vol \cone(C_{max},y)} = \left(\frac{L(y) - L(x_{\min})}{L(y) - L(x_{\max})}\right)^d = \left(1 + \frac{L(x_{\max}) - L(x_{\min})}{L(y) - L(x_{\max})}\right)^d \leq \left(1 + \frac{2}{d}\right)^d \leq \euler^2
\] 
Further, since $C$ is sandwiched between the two cones we have $\ConvHull(C) \subset \cone(C_{min},y) - \cone(C_2,y)$, and can bound the ratio in volume between the Convex Hull of $C$ and the smaller cone:
\begin{align*}
    \vol \ConvHull(C) &\leq \vol \cone(C_{min},y) - \vol \cone(C_{max},y)\\
    \frac{\vol \ConvHull(C)}{\vol \cone(C_{max},y)} &\leq \frac{\vol \cone(C_{min},y)}{\vol \cone(C_{max},y)} - 1\\ 
    &\leq \euler^2 - 1
\end{align*}
Finally, because the Convex Hull of $C \cup \{y\}$ contains both Cone($C_{max}, y$) and ConvHull($C$), this allows us to lower bound the expansion factor of including $y$ into $C$:
\begin{align*}
    \frac{\vol \ConvHull(C,y)}{\vol \ConvHull(C)}  &\geq \frac{\vol \cone(C_{max},y) + \vol \ConvHull(C)}{\vol \ConvHull(C)}\\
    &= \frac{\vol \cone(C_{max},y)}{\vol \ConvHull(C)} + 1\\
    &\geq \frac{1}{\euler^2-1} + 1\\
    &\geq \frac{\euler^2}{\euler^2-1}
\end{align*}
\\
\\
Using Equation \eqref{claim2}, we can build our contradiction on the volume of the convex hull for large enough $C$. For analysis, we denote the size of $C$ by $n$. To start, we note a simple upper bound on the volume of the convex hull of any $n$ point set $C \in \mathbb{R}^d$:
\[
\vol \ConvHull(C) \leq V_{max}{\binom{n}{d}},
\]
where $V_{max}$ is the volume of the largest simplex with vertices in $C$. This follows from choosing any vertex $x \in C$ and noting that choosing every simplex which contains $x$ is a triangulation of ConvHull($C$). While this triangulation is certainly not optimal, it is sufficient for our purposes.
\\
\\
Since there exists some $h$ s.t. no point in $C$ can be inferred from the rest, every point added to $C$ after the largest simplex multiplies the volume by $\frac{e^2}{e^2-1}$. This gives a lower bound on the volume of ConvHull($C$) of:
\[
\vol \ConvHull(C) \geq V_{max}\left (\frac{e}{e-1} \right)^{n-d-1}
\]
Together, these bounds give the equation:
\[
\left (\frac{e^2}{e^2-1} \right)^{n-d-1} \leq {\binom{n}{d}}
\]
Setting $n>24d\log(d+1)$ gives a contradiction. 
\end{proof}
\noindent With Lemmas \ref{eq:slotting}, \ref{cluster:label}, and \ref{cluster:id} in hand, we can now give a more detailed explanation of step 2a:
\paragraph{Step 2a (high noise):} It is assumed by step 1 that we have detected an $\varepsilon$-equitable subset $S'$. Draw an additional set of points $\{x_1,\ldots,x_m\}$, and for each point test whether $S' \cup \{x_i\}$ is $\np_1$-equitable. By Lemma~\ref{eq:slotting}, the points which measure as equitable with $S'$ make up a cluster. Using Lemma~\ref{cluster:label} to label these points, build an LP based on the labels and cluster structure. Applying Lemma~\ref{cluster:id} and the symmetry argument of \cite{kane2017active} shows that this LP has good coverage.
\\
\\
It is left to show that step 2b has good coverage. Step 2b will follow a similar strategy to the Massart case, using points well-separated in an MLE ordering to build our LP. However, since we are still in the regime of unbounded error, we will need to exploit the fact that our sample has no large clusters to show that this LP infers correctly with high probability. Notice that a sample with no clusters consists mostly of pairs of points whose comparisons are bounded in error. With this in mind, we modify the pointwise movement bounds of \cite{braverman2009sorting} to differentiate between pairs of points with bounded and unbounded comparison error.
\begin{definition}
Let $S$ be a set with a noisy comparison oracle $Q_C$. We call a comparison between points $x,y \in S$ $\np$-far if the probability that $Q_C$ returns the correct comparison is at least $1/2+\np$. Otherwise we call the comparison $\np$-close.
\end{definition}
To prove a point-wise movement bound, we will follow exactly the strategy of \cite{braverman2009sorting}. First, we prove that it is unlikely that an ordering which disagrees on many far comparisons from the true order is an MLE ordering. Second, we use this to upper bound the total number of wrong far comparisons in any MLE order with high probability. Finally, we prove that as long as no large cluster exists, a single point cannot move too far without contradicting the upper bound on total far errors.
\begin{lemma}
\label{BM:8}
Let $\sigma$ be a permutation which differs from the true order on $\sigma_c$ $\np$-close comparisons, and $\sigma_f$ $\np$-far comparisons. The probability that $\sigma$ is an MLE order is
\[
Pr[\sigma \ \text{is MLE}] \leq e^{-\frac{\np^2\sigma_f^2}{2\sigma_f+\sigma_c}}
\]
\end{lemma}
\begin{proof}
To be an MLE order, $\sigma$ must beat the true order on half or more of the comparisons on which they differ. We can bound this probability by the Poisson Binomial:
\[
Pr\left[Bin(1/2,\sigma_c) + Bin(1/2+\np,\sigma_f) \leq \frac{\sigma_c+\sigma_f}{2}\right].
\]
A Chernoff bound then gives the desired result.
\end{proof}
Using this upper bound, we show that any order which disagrees with the true ordering on more than $\tilde{\Omega}(n^{3/2})$ comparisons is not an MLE ordering with high probability. 
\begin{lemma}[Total Far Movement]
\label{avg:far}
The probability that an MLE order disagrees with the identity on $c_1n^{3/2}$ $\np$-far comparisons, where
\[
c_1 = \sqrt{\frac{\log(1/\delta)+n\log(n)}{\np^2 n}},
\]
is $\leq \delta$
\end{lemma}
\begin{proof}
For a given permutation $\sigma$, assume $\sigma_f > c_1n^{3/2}$. By Lemma~\ref{BM:8}, the probability that $\sigma_f$ is an MLE order is at most:
\[
Pr[\sigma \ \text{is an MLE}] \leq e^{-\frac{c_1^2n^3\np^2}{2\sigma_f+\sigma_c}} \leq e^{-c_1^2n\np^2}
\]
To get the probability that there exists such a $\sigma$ that is an MLE order, we union bound over all permutations, giving:
\[
Pr[\exists \sigma : \sigma_f > cn^{3/2} \land \sigma \ \text{is an MLE order}] \leq 2^{n\log(n)}e^{-c_1^2n\np^2} \leq \delta
\]
\end{proof}
Finally, we show a bound on point-wise movement by proving that any point which moves more than $\tilde{\Omega}(n^{3/4})$ from its true position creates $\tilde{\Omega}(n^{3/2})$ total far errors.
\begin{lemma}[Point-wise Far Movement]
\label{pointwise}
Given a sample $S$ of size $n$ and $\np \leq g_L(\varepsilon_0)$, assume that the sample does not have a $g_L^{-1}(\np)$-cluster of  size $m$. Let $l = (2c_1)^{1/2}n^{3/4}$. Then with probability at least $1-2\delta$, no point moves by further than $c_2m_2$ in an MLE order, where
\[
c_2 = 5/\np, m_2 = \max \left(m,l,\frac{20\log(n^2/\delta)}{\np} \right)
\]
\end{lemma}
\begin{proof}
Assume without loss of generality that the true order on $S$ is the identity $1,\ldots,n$. Denote by $A_{ij}$ the event that $i$ maps to $\sigma(i)=j$ in an MLE order, $|i-j|>c_2m_2$, and at most $l$ elements from outside the range $[i-l-m,j+l+m]$ map into $[i,j]$. Note that if more than $l$ of such elements map into $[i,j]$ then the order must differ on at least $\frac{l^2}{2}$ $\np$-far comparisons from the identity. This follows from the fact that each such element must shift $m+l$ places towards $[i,j]$, but has at maximum $m$ $\np$-close comparisons in that direction, and that each comparison is counted at most twice.
\\
\\
For $i$ to be in slot $j$ in an MLE order, it must beat the identity on more than half of elements in between. Since we have assumed all but $l$ of the elements between $i$ and $j$ in the order are from $[i-l-m,j+l+m]$, then this range must contain at least $c_2m_2/2-l-1$ incorrect comparisons with $i$. This further implies that at least $c_2m_2/2-2l-m-1$ comparisons with $i$ must be incorrect in the range $[i,j+l+m]$. By our assumption on cluster size, all but $m$ of these comparisons are $\np$-far, so we can bound the probability of $A_{ij}$ by the Poisson Binomial:
\[
Pr[Bin(1/2,m)+Bin(1/2-\np,c_2m_2+l) > c_2m_2/2-2l-m-1
\]
Combining our assumptions on $m_2$ with a Chernoff bound then gives:
\[
Pr[A_{ij}] \leq e^{-\frac{\np m_2}{20}} \leq \frac{\delta}{n^2}.
\]
Union bounding over pairs $i,j$ then gives that if any point moves by more than $c_2m_2$ in an MLE ordering, the total number of wrong $\np$-far comparisons are more than $c_1n^{3/2}$ with probability $1-\delta$. By Lemma~\ref{avg:far}, the probability that this occurs is $\leq \delta$, giving the desired result.
\end{proof}
With a point-wise movement bound in hand, step 2b essentially follows the same strategy as Lemma~\ref{lemma:weak-learner-massart} with a different set of parameters. 
\paragraph{Step 2b (low noise):} Draw an additional sample of $m$ points, and use the labels and comparisons of all pairs of points separated by $\tilde{\Omega}(n^{3/4})$ in $S$ to build an inference LP. This LP correctly infers points with high probability by Lemma~\ref{pointwise}, and has large coverage due to the space's finite inference dimension.
\\
\\
All that remains is step 3, which repeats steps 1 and 2 until reaching the desired coverage. Tying all of these together, we present the proof of Theorem~\ref{TNC}: learning margin $\gamma$, finite inference dimension non-homogeneous linear separators with GTNC noise.

\begin{proof} (Proof of Theorem~\ref{TNC})
\\
Let $S$ be the subsample described in step 1 of size $n$, and $c$ and $m$ the parameters defining the size of subsets we check for $\varepsilon_T$-equitability. Further, in the case that some subset tests as equitable, let $S'$ be the additionally drawn points. To begin, we set $\varepsilon_T$ such that if we measure an equitable subset $S_{eq} \subset S$, points $x\in S'$ s.t. $S_{eq} \cup \{x\}$ is $2g_U(2g_L^{-1}(4\varepsilon_T))$-equitable make up a $\gamma/d$-cluster (see Lemma~\ref{eq:slotting}):
\begin{align}
\label{eqn:equitable-constants}
\varepsilon_T &\coloneqq \left (\frac{g_L \left(\frac{g_U^{-1}\left(\frac{g_L\left(\gamma' \right)}{4} \right)}{2} \right)}{4} \right ), \gamma' = \min\left(\frac{\gamma}{2d},\frac{\varepsilon_0}{2}\right)
\end{align}
Note that this also satisfies the requirement on $\varepsilon_T$ from Lemma~\ref{eq:slotting}. To satisfy Lemmas \ref{cor:equitable}, \ref{eq:slotting}, and \ref{cluster:label}, we set $c$, $m$, and $|S'|$ to:
\[
c = \frac{48\log(n)+\log(1/\delta_r)}{\varepsilon_T^2},m= c_1^{1/2}n^{3/4}, |S'| = d\log(d+1)n, c_1 = \frac{\sqrt{\log(n/\delta_r)}}{g_L(g_U^{-1}(\varepsilon_T/2)} 
\]
Note that $c_1$ is a simplified (and somewhat larger) version of the parameter from Lemma~\ref{avg:far} where $\np$ has been set to $(g_L(g_U^{-1}(\varepsilon_T/2))$. We must further set parameters $c_2$ and $m_2$ to satisfy Lemma~\ref{pointwise}:
\begin{align*}
c_2 = \frac{5}{g_L(g_U^{-1}(\varepsilon_T/2))},m_2 = O(m)
\end{align*}
Finally, we must select the sample size $n$ itself. To employ the same slotting strategy as Theorem \ref{massart}, we need $\Omega(k)$ blocks of size $c_2m_2$. This gives the requirement on $n$:
\begin{align*}
n=\Omega\left( kc_2m_2 \right) &=\Omega\left( \frac{kc_1^{1/2}n^{3/4}}{g_L(g_U^{-1}(\varepsilon_T/2))}\right)\\
\implies n &\geq \Omega \left (\frac{k^4\log(n/\delta_r)}{g_L(g_U^{-1}(\varepsilon_T/2))^6} \right)
\end{align*}
To satisfy this condition, it is enough let $n$ be:
\[
n = \theta \left (\frac{k^4\log\left(\frac{k}{g_L(g_U^{-1}(\varepsilon_T/2))\delta_r}\right)}{g_L(g_U^{-1}(\varepsilon_T/2))^6} + \log^{4/3}(d) \right),
\]
where the additional factor in $d$ ensures that $m$ and $m_2$ satisfy the constraints of Lemmas~\ref{eq:slotting} and \ref{pointwise}.
\\
\\
We will now structure our analysis as in the 3 step informal explanation.
\paragraph{Step 1:} Draw the sample $S\sim D^n_X$, where in later iterations $D$ is restricted to un-inferred points by rejection sampling. Check $S$ for $\varepsilon_T$-equitable subsets of size $2c+m$.
\paragraph{Step 2a (high noise):}
Assume that at least one subset, $S_{eq}$, is equitable with true cluster $C$. Draw an additional set $S'$ and test for each $x \in S'$ whether $S_{eq} \cup x$ is $\frac{g_L(\gamma')}{2}$-equitable. With probability $1-\bigo(\delta_r)$, we can identify by Lemma~\ref{eq:slotting} and correctly label by Lemma~\ref{cluster:label} at least $96d\log(d+1)+\frac{2\log(1/\delta_r)}{g_L(\gamma)^2}$ points of $S'$ which are in a $\gamma/d$ cluster. We build our learner based off of this cluster. Recall that the expected coverage of the learner is given by the probability that an additional point is inferred. To compute this, we first note that the probability an additional point lands inside the cluster is at least $\Omega(m/n)$. Assuming this occurs, Lemma~\ref{cluster:id} and the symmetry argument of \cite{kane2017active} give the point a $3/4$'s probability of being inferred. Together with our high probability assumptions, this gives an expected coverage of $\Omega(m/n)$ for small enough $\delta_r$. Thus, the probability that the coverage of our weak learner is $\Omega(m/n)$ is at least $\Omega(m/n)$ by the Markov inequality.
\paragraph{Step 2b (low noise):}
Assume instead that no subset was $\varepsilon_T$-equitable. By statement 1 of Corollary \ref{cor:equitable}, this implies that no $g_U^{-1}(\varepsilon_T/2)$-cluster of size $2c+m$ exists in $S$. Sort $S$ into an MLE order. By Lemma~\ref{pointwise}, no point in $S$ has moved by further than $c_2m_2$ from its true position with probability at least $1-\delta_r$. $S$ is of the appropriate size to apply the argument from Lemma~\ref{slot-bound}, so slotting $\bigo(k)$ extra points gives constant coverage with constant probability.
\paragraph{Step 3:}
Steps 1 and 2 build a weak learner which we must string together to get coverage $1-\varepsilon$ mirroring Theorem~\ref{massart}. Our worst case per-step coverage is $\Omega(m/n)$ with probability $\Omega(m/n)$. After repeating the learner $t$ times, the coverage becomes:
\[
Pr\left[\text{Coverage} > \Omega(m/n)\right] \geq (1-\Omega(m/n))^t.
\]
Denoting the reliability and usefullness parameters again as $\delta_r^w$ and $\delta_u^w$, setting $t = \tilde{\bigo} \left(\frac{n\log(1/\delta_u^w)}{m}\right)$ is then sufficient to give this coverage with probability at least $1-\delta_u^w$.
\\
\\
Restricting to the distribution of un-inferred points via rejection sampling, repeating the above $\bigo \left(\frac{n\log(1/\varepsilon)}{m}\right)$ times will have coverage $1-\varepsilon$ with probability $1-\bigo \left (\frac{n\log(1/\varepsilon)}{m}\delta_u^w \right )$, and correctness $1-\bigo \left (\frac{n^2\log(1/\varepsilon)\log(1/\delta_u^w)}{m^2}\delta_r^w\right )$. Thus setting $\delta_r^w$ and $\delta_u^w$ of our weak learner to:
\begin{align*}
\delta_u^w &\to \bigo \left (\frac{m\delta_u}{n\log(1/\varepsilon)}\right)\\
\delta_r^w &\to \bigo \left ( \frac{\delta_r}{\alpha\log\left(\frac{1}{\delta_u'}\right)} \right )\\
\alpha &= \frac{n^2\log(1/\varepsilon)}{m^2},
\end{align*}
gives the desired coverage and error by union bounding over the number of applications.
\paragraph{Query Complexity:} Now we compute the Query complexity of our algorithm. Because we check equitability for every subset, at each iteration our algorithm must make $\bigo(n^2)$ comparisons. This is dominated by the slotting complexity, which we upper bound as $\tilde{\bigo}(dn^2)$ for simplicity. The worst-case number of iterations for our algorithm is $\alpha\log(\alpha/\delta_u),$ giving a total query complexity of:
\[
q(\varepsilon,\delta_r,\delta_u) = \tilde{\bigo}\left(\frac{n^{5/2}}{c_1}d\log(1/\varepsilon)\log(1/\delta_u)\right)
\]
For sample complexity, we follow the same argument of Theorem~\ref{massart}, ending our algorithm if we reject too many samples in a row. Letting $N=\bigo\left( d\log(d)n\alpha\log\left(\frac{\alpha}{\delta_u}\right)\right )$, the sample complexity is then:
\[
\bigo\left(\frac{N\log(N/\delta_u)}{\varepsilon} \right)
\]
Our time complexity, however, diverges from the Massart case due to our need to test all subsets for equitability. In particular, we check all ${\binom{n}{2c+m}}$ subsets, which is exponential in inference dimension and noise parameters, and quasi-polynomial in the error parameter $\delta_r$. Further, with unbounded error we cannot employ the sorting algorithm from \cite{braverman2009sorting}, making sorting an exponentially expensive step as well.
\end{proof}
As a direct corollary, we show that this gives us a query efficient\footnote{query efficiency for $\gamma^{-1} = \polylog(1/\varepsilon)$} algorithm for the special case of TNC.
\begin{corollary}
Let the hypothesis class $(X,H_{d,\gamma})$ have inference dimension $k$. Then $(X,H_{d,\gamma})$ is ARPU-learnable under model (TNC$(m,M,\kappa,\varepsilon_0)$,$\mathcal{C}_X$) with query complexity:
\begin{align*}
q(\varepsilon,\delta_r,\delta_u) &=\tilde{\bigo} \left (\frac{k^{10}M^{28}2^{14\kappa}}{m^{42}\gamma'^{14(\kappa-1)}}d\log^{2}\left( \frac{1}{\delta_r} \right)\log\left(\frac{1}{\varepsilon}\right)\log\left(\frac{1}{\delta_u}\right)\right ).\\
\end{align*}
\end{corollary}
As an example of an explicit concept class, consider the query complexity of half-spaces with fixed minimal-ratio (the ratio between the closest and furthest points from the decision boundary), a case studied in \cite{kane2017active}.
\begin{example}
Let $X \subseteq \mathbb{R}^d$ be an instance space, and $H_{d,\gamma,\eta}$ the class of hyperplanes with margin $\gamma$ and minimal ratio $\eta$ with respect to $X$. Then $(X,H_{d,\gamma,\eta})$ is ARPU-learnable under model $(TNC(m,M,\kappa,\varepsilon_0),\mathcal{C}_X)$ with query complexity:
\[
q(\varepsilon,\delta_r,\delta_u) = \poly\left(d,\frac{1}{\varepsilon_0},\frac{1}{\gamma},\log\left(\frac{1}{\eta}\right),\log\left(\frac{1}{\delta_r}\right),\log\left(\frac{1}{\delta_u}\right),\log\left(\frac{1}{\varepsilon}\right) \right )
\]
\end{example}

\subsection{GTNC with Weak Distributional Conditions}

Our algorithm for learning with GTNC noise introduced an additional restrictive condition on the set system: margin $\gamma$. We will show that this assumption and the assumption of finite inference dimension may be replaced with weak concentration and anti-concentration conditions on the distribution. In this case, however, it is difficult to show a gap between label only and comparison ARPU-learning for two reasons. The first is that learning in this regime in simply harder--it is the first case we show where comparisons do not provide an exponential improvement in the active PAC setting over its passive counterpart. The second is that in the membership query setting, label queries in the TNC model can give comparison like information, making it difficult to apply our lower bounding techniques. We will begin by proving this first statement by showing a lower bound polynomial in $\varepsilon^{-1}$ for active PAC learning with labels and comparisons.
\begin{lemma}
	\label{lemma:lower-bound-gtnc}
	Let $s = \min \left (1, g_L^{-1}(1/8) \right )$, and
	$c_1 = \underset{a \in [2\varepsilon,s]}{\max}(8g_L(4\varepsilon),2(g_L(a)-g_L(a-2\varepsilon)))$.
	The query complexity of actively PAC-learning $(\R^2, H_2)$ under model $(\gtnc(g_L, g_U, \varepsilon_0), \mathcal{SC}_{2})$ is at least 
	\[
	q(\varepsilon, 1/8) = \Omega\left(\frac{1}{c_1}\right)
	\] 
	for $\varepsilon \leq \frac{g_L^{-1}(1/16)}{4}$.
\end{lemma} 
\begin{proof}
	The adversary begins by choosing the distribution over $\mathbb{R}^2$ to be uniform over the square $S=[0,s]^2$. We will use $(a,b)$ to denote points in $\R^2$. Consider two parallel hyperplanes $h,h_\varepsilon$ defined as: 
	\begin{align*}
	h\text{: } a =0 \quad \text{and} \quad h_\varepsilon\text{: } a =2\varepsilon.
	\end{align*}
	We denote the region between the two hyperplanes by $\Delta := \{(a,b)\in S: 0\leq a\leq 2\varepsilon\}$, and twice the region as $2\Delta := \{(a,b): 0\leq a\leq 4\varepsilon \}$
	
	By Yao's minimax principle it is enough to show that the adversary may pick a distribution over hyperplanes such that no learner can learn the labels with $<\varepsilon$ error with probability $\geq 7/8$. In particular, the adversary considers a uniform distribution over hyperplanes $h$ and $h_\varepsilon$. Note that any algorithm which correctly labels more than half of the points between $h$ and $h_\varepsilon$ (i.e. at least $\varepsilon$ mass of $S$) can be seen as identifying the hyperplane $h$ or $h_\varepsilon$. We now show how to lower bound the number of label or comparison queries needed to identify the target hyperplane $h$ or $h_\varepsilon$.

	Given a set of $n$ query responses $Q_1,\ldots,Q_n$ from the learner, we argue that the learner cannot succeed with probability greater than:
	\[
	\max(P(h | Q_1,\ldots,Q_n),P(h_\varepsilon | Q_1,\ldots,Q_n)),
	\]
	since it can do no better than simply picking the more likely hyperplane given the set of queries. Taking the maximum over all possible sets of query responses then gives a lower bound on the number of samples. In other words, to show that the learner must make at least $n$ queries, it suffices to show that this maximum is less than $7/8$:
	\begin{align}\label{eq:TNC-LB}
	\max_{Q_1,\ldots,Q_n} \left (P(h | Q_1,\ldots,Q_n),P(h_\varepsilon | Q_1,\ldots,Q_n) \right ) < 7/8
	\end{align}
	Using Bayes theorem, we can rewrite these probabilities as:
	\begin{align*}
	P(h | Q_1,\ldots,Q_n) &= \frac{1}{1 + \prod_{i=1}^n\frac{P(Q_i|h_{\varepsilon},Q_{i-1},\ldots,Q_1)}{P(Q_i|h,Q_{i-1},\ldots,Q_1)}}\\
	P(h_{\varepsilon} | Q_1,\ldots,Q_n) &= \frac{1}{1 + \prod_{i=1}^n\frac{P(Q_i|h,Q_{i-1},\ldots,Q_1)}{P(Q_i|h_\varepsilon,Q_{i-1},\ldots,Q_1)}}
	\end{align*}
	Note in this case that query response $Q_i,$ which rolls together both the point or pair of points being queried and the value which the oracle returns, is dependent on $Q_{i-1},\ldots,Q_{1}$ due to being in an active setting--the chosen point or pair is dependent on the previous responses $Q_{i-1},\ldots,Q_{1}$. We can now rewrite Equation~\eqref{eq:TNC-LB} as:
	\[
	\forall Q_1,\ldots,Q_n: 7 > \prod_{i=1}^n\frac{P(Q_i|h,Q_{i-1},\ldots,Q_1)}{P(Q_i|h_\varepsilon,Q_{i-1},\ldots,Q_1)} > 1/7.
	\]
	To analyze this, note that each term in the product is simply the ratio of probabilities that a label query on some point $x$ or comparison on pair of points $x,y \in S$ (where $x,y$ are determined by $Q_{i-1},\ldots,Q_1$) will return $Q_i$. Then we can bound this product from above and below by looking at the maximum and minimum such ratio across all points and pairs in $S$. Recall that these probabilities are chosen by the adversary from a range defined by the GTNC parameters. For simplicity, when the ranges on a query for $h$ and $h_\varepsilon$ overlap, we let the adversary choose the same probability, but otherwise always choose the lower bound $g_L$.

	To begin, we consider maximizing the ratio. In this case we only need to consider $Q_i$ as the correct label or comparison for $h$, as this will always have the larger ratio. For a point $(a,b) \in S$, the ratio for the correct label $(Q_i=+)$ for $h$ is given by:
	\[
	\begin{cases}
	\frac{1/2+g_L(a)}{1/2-g_L(2\varepsilon-a)} \leq 1+8g_L(2\varepsilon) & (a,b) \in \Delta \\
	\frac{1/2+g_L(a)}{1/2+g_L(a-2\varepsilon)} \leq 1 + 2(g_L(a)-g_L(a-2\varepsilon)) & (a,b) \in  S \setminus \Delta
	\end{cases}
	\]
	For comparisons, we only have to consider pairs $(a_1,b_1),(a_2,b_2) \in 2\Delta$, since the adversary will otherwise pick a ratio of $1$. In this case, the maximum is given by the correct comparison with ratio:
	\[
	\frac{1/2+g_L(|a_1-a_2|)}{1/2-g_L(|a_1-a_2|)} \leq 1+8g_L(4\varepsilon)
	\]
	Thus we can bound the product of the ratios from above by: 
	\[
	\forall Q_1,\ldots,Q_n: \prod_{i=1}^n\frac{P(Q_i|h,Q_{i-1},\ldots,Q_1)}{P(Q_i|h_\varepsilon,Q_{i-1},\ldots,Q_1)} \leq \underset{a \in [2\varepsilon,s]}{\max}((1+8g_L(4\varepsilon))^n, (1 + 2(g_L(a)-g_L(a-2\varepsilon)))^n)
	\]
	To bound the ratio from below, we look at the probability for the incorrect label or comparison. For labels, this is:
	\[
	\begin{cases}
	\frac{1/2-g_L(a)}{1/2+g_L(2\varepsilon-a)} \geq 1-4g_L(2\varepsilon) & (a,b) \in \Delta \\
	\frac{1/2-g_L(a)}{1/2-g_L(a-2\varepsilon)} \geq 1-4(g_L(a)-g_L(a-2\varepsilon))& (a,b) \in  S \setminus \Delta
	\end{cases}
	\]
	Likewise, the minimum ratio for comparisons is: 
	\[
	\frac{1/2-g_L(|a_1-a_2|)}{1/2+g_L(|a_1-a_2|)} \geq 1-4g_L(4\varepsilon)
	\]
	Thus we can also bound the product of the ratios from below as: 
\[
\forall Q_1,\ldots,Q_n: \prod_{i=1}^n\frac{P(Q_i|h,Q_{i-1},\ldots,Q_1)}{P(Q_i|h_\varepsilon,Q_{i-1},\ldots,Q_1)} \geq \underset{a \in [2\varepsilon,s]}{\min}((1-4g_L(4\varepsilon))^n, (1 - 4(g_L(a)-g_L(a-2\varepsilon)))^n)
\]

	Let $c_1 = \underset{a \in [2\varepsilon,s]}{\max}(8g_L(4\varepsilon),4(g_L(a)-g_L(a-2\varepsilon)))$, then it is sufficient to pick $n$ such that:
	\[
	(1-c_1)^n > 1/7 \ \text{and} \ (1+c_1)^n < 7
	\]
	Recalling that $c_1 < 1/2$ due to the initial values of $s$ and $\varepsilon$, setting $n$ to:
	\[
	n = \frac{\log(7)}{2c_1}
	\]
	satisfies this and in turn Equation~\eqref{eq:TNC-LB}, completing the proof.
\end{proof}
Note that for notational simplicity the adversary has chosen a non-isotropic distribution, but the bound is easily modified to hold for a distribution in $\mathcal{ISC}_2$. Specifying to the Tsybakov Low Noise condition gives the following lower bound.
\begin{corollary}[Restatement of Lemma~\ref{intro:dd:TNC-lower-bound}]\label{dd:TNC-lower-bound}
	The query complexity of actively PAC-learning $(\R^2, H_2)$ under model $(\tnc(m,M,\kappa,\varepsilon_0), \mathcal{SC}_2)$ is at least 
	\[
	q(\varepsilon, 1/8) = \Omega\left(\frac{1}{\max\{\epsilon, \epsilon^{\kappa - 1}\}}\right)
	\] 
	where $\varepsilon \leq \frac{\left(\frac{1}{16m}\right)^{\frac{1}{\kappa-1}}}{4}$.
\end{corollary}
\begin{proof}
	Observe that for $f(x) = m x^{\kappa - 1}$, $|\nabla f(x)| \leq m (\kappa - 1)$ for all $x \in [0,s]$. By the mean value theorem, \[
		|f(x) - f(y)| \leq  m (\kappa - 1) |x - y| \quad \forall ~ x,y \in [0,s].
	\] 
	Specifying to the TNC model from GTNC, we have $g_L(x)=f(x)$, and thus that $g_L(x) - g_L(x - 2\varepsilon) = f(x)-f(x-2\varepsilon) \leq 2 m (\kappa - 1) \varepsilon$ for $x \in [2\varepsilon, s]$, and $8g_L(4\varepsilon)=\Theta(\varepsilon^{\kappa-1})$. Plugging this into Lemma \ref{lemma:lower-bound-gtnc} then gives the desired bound.
\end{proof}
Note that this bound is tight with respect to $\varepsilon$ for $\kappa >2$, and not far off for $1 < \kappa < 2$, as Hanneke and Yang \cite{hanneke2015minimax} provide a label only active PAC-learning algorithm with $\tilde{O}_d(\frac{1}{\varepsilon})$ queries and $\tilde{O}_d( \left (\frac{1}{\varepsilon} \right)^{2-2/\kappa})$ queries respectively. However, while comparison queries alone may not enough to exponentially improve the query complexity over passive PAC-learning (which is also polynomial in $\varepsilon^{-1}$ \cite{massart2006risk}), we will show that they are sufficient for ARPU-learning.
\begin{theorem}[Restatement of Theorem~\ref{intro:TNC:aid}]
\label{TNC:AID}
The hypothesis class $(\R^d,H_d)$ is ARPU-learnable under model $(\gtnc(g_L,g_U,\varepsilon_0), \mathcal{ACC}_{d,c_1,c_2})$ with query complexity:
\[
q(\varepsilon,\delta_r,\delta_u) =\tilde{\bigo} \left (\frac{d^{11}}{\left(g_L \circ G_8 \circ \frac{G_2\circ\frac{G_4(\varepsilon')}{4d}}{2}\right)^{14}}
\log^{2}\left( \frac{1}{\delta_r} \right)\log\left(\frac{1}{\delta_u}\right)\log^2\left(\frac{1}{\varepsilon}\right)\right )
\]
for small enough $\delta_r$, where
\begin{align*}
\varepsilon' = \min\left(\frac{\varepsilon}{4c_2},\frac{\varepsilon_0}{2}\right), G_c(x) = g_U^{-1}\left( \frac{g_L(x)}{c}\right).
\end{align*}
\end{theorem}
The margin condition is necessary for Lemmas~\ref{cluster:label} and \ref{cluster:id}--we cannot reliably label points or infer from clusters lying close to the decision boundary. If we were only interested in keeping our guarantee on coverage, it would be enough to set a fake margin $\gamma$ such that anti-concentration gives that the set of points with such a margin has $ O(\varepsilon)$ probability mass. However, we also require that our algorithm is reliable, and thus with high probability cannot err on points close to the decision boundary. This suggests the following strategy: if a cluster is found in step 1, before using it for inference, test whether it is too close to the decision boundary. Because the error on our labels is proportional to their distance from the decision boundary, we can build a test similar to Lemma~\ref{test} to detect this by measuring the relative sizes of the subsets with different labels.
\begin{lemma}[Margin Detection]
\label{margin-test}
Let $C$ be a $\gamma/d$-cluster with respect to the hyperplane $f$ of size at least $\frac{16\log(4/\delta)}{g_U(2\gamma)^2}$, and
\[
\gamma<\frac{g_U^{-1}\left(\frac{g_L(\varepsilon_0)}{4}\right)}{2}.
\] 
Further, let $L_{\text{Dif}}(C)$ denote the difference in size between the sets $\{x \in C: Q_L(x)=1\}$ and $\{x \in C: Q_L(x)=0\}$. With probability at least $1-\delta$:
\begin{enumerate}
\item If $\exists x \in C$ with $f(x)<\gamma$, then $|L_{\text{Dif}}(C)| < (1/2+2g_U(2\gamma))|C|$
\item If $\forall x \in C$, $f(x)>g_L^{-1}(4g_U(2\gamma))$, then $|L_{\text{Dif}}(C)| \geq (1/2+2g_U(2\gamma))|C|$
\end{enumerate}
\end{lemma}
\begin{proof}
Assume without loss of generality that the true label of the majority of points in $C$ is 1. 
\\
\\
\noindent Proof of (1): If there exists a point $x \in C$ with $f(x)<\gamma$, then the entire entire cluster lies within margin $\gamma+\gamma/d<2\gamma$. By assumption $2\gamma<\varepsilon_0$, so the probability that a point measures as 1 is at most $1/2+g_U(2\gamma)$ using Equation \eqref{eqn:gtnc-label-1}. The probability that more than $(1/2+2g_U(2\gamma))|C|$ points label as 1 is then given by a Chernoff bound:
\[
Pr[L_{\text{Dif}}(C) \geq (1/2+2g_U(2\gamma))|C|-1] \leq e^{-\frac{g_U(2\gamma)^2|C|}{16}} \leq \delta/4
\]
Since we have assumed the majority label is 1, the probability that more than $(1/2+2g_U(2\gamma))|C|$ label as 0 is upper bounded by this as well.
\\
\\
Proof of (2): Assume $\forall x \in C$ we have $f(x)<g_L^{-1}(4g_U(2\gamma))$. Since $g_L^{-1}(4g_U(2\gamma)) < \varepsilon_0$ by assumption, the probability that any point measures as 1 is at least $1/2+4g_U(2\gamma)$ using Equation \eqref{eqn:gtnc-label-1}. The probability that less than $(1/2+2g_U(2\gamma))|C|$ points label as 1 is then given by a Chernoff bound:
\[
Pr[L_{\text{Dif}}(C) \leq (1/2+2g_U(2\gamma))|C|] \leq e^{-\frac{g_U(2\gamma)^2|C|}{2}} \leq \delta/2
\]
\end{proof}
The idea is now to follow the structure of Theorems~\ref{TNC} and \ref{massart:aid} with the one exception that we will check the closeness of every cluster to the decision boundary by checking whether $|L_{Dif}(C)| \geq (1/2+2g_U(2\gamma))$. If a cluster measures as too close, we will avoid labeling the points, preserving the reliability of the algorithm.
\begin{proof}(Proof of Theorem~\ref{TNC:AID})
\\
To ensure that our coverage is wide enough, we will need to set the margin parameter such that for any hyperplane, the probability mass of points within margin $2g_L^{-1}(4g_U(2\gamma))$ is at most $\varepsilon/2$. By our anti-concentration bound, it is enough to let $\gamma$ be:
\[\gamma = \frac{g_U^{-1}\left(\frac{g_L\left(\varepsilon'\right)}{4}\right)}{2}, \varepsilon' = \min\left(\frac{\varepsilon}{4c_2},\frac{\varepsilon_0}{2}\right).
\]
Our goal is to learn the rest of the space up to $\varepsilon/2$ error via Theorem~\ref{TNC}, assuming for the moment that the modification from Lemma~\ref{margin-test} will cause at most an overall loss of $\varepsilon/2$ coverage. Noting that our space has good average inference dimension, i.e. $\mathcal{ACC}_{d,c_1,c_2} \subset \mathcal{A}_{(X,\mathcal{H}),1}$ \cite{AID}, we will achieve this by applying Lemma~\ref{sample}. Thus we need to prove that Thereom~\ref{TNC} can be used to learn a $(1-\varepsilon/6)$ fraction of random samples $S$ with probability at least $(1-\varepsilon/6)$ while querying only $(1-\varepsilon/6)$ points.
\\
\\
To begin, we must set $\varepsilon_T$ to detect $\gamma/d$-clusters: 
\begin{align*}
4\varepsilon_T &= g_L \left(\frac{g_U^{-1}\left(\frac{g_L\left(\frac{g_U^{-1}\left(\frac{g_L\left(\varepsilon'\right)}{4}\right)}{4d} \right)}{2} \right)}{2} \right),
\end{align*}
and set the size of $S$ such that the algorithm in Theorem~\ref{TNC} only queries an $\varepsilon/6$ fraction of points.
Letting $N$ be the total number of points queried as given in Theorem~\ref{TNC}, it is then sufficient for $|S|$ to satisfy:
\begin{align*}
|S|&= \Theta \left(\frac{N}{\varepsilon}\right)\\
N &= \tilde{\bigo} \left (\frac{k^6}{g_L(g_U^{-1}(\varepsilon_T/2))^8}\log^{2}\left(\frac{1}{\delta_r}\right)\log^2\left(\frac{1}{\varepsilon}\right) \right ).\\
\end{align*}
Note that due to the distributional conditions, the inference dimension $k$ of our sample is $O(d\log(d)\log(|S|))$ with probability at least $1-\varepsilon/12$ \cite{AID}. Applying the same argument from Theorem~\ref{massart:aid} then gives that the learner of Theorem~\ref{TNC} satisfies the conditions of Lemma~\ref{sample}. Thus to have coverage $1-\varepsilon/2$ with probability $1-\delta_u$ and reliability $1-\delta_r$, it is sufficient to set our $\delta_r$ to $\bigo\left (\frac{\delta_r}{\log(1/\delta_u)} \right)$ and run the algorithm from Theorem~\ref{TNC} $\bigo(\log(1/\delta_u))$ times.
\\
\\
We have ignored, up until now, the modification to Theorem~\ref{TNC} in the cluster step. If a subset measures as equitable, after slotting our extra points to obtain the $\gamma/d$-cluster $C$, we use Lemma~\ref{margin-test} to test the margin of $C$. If the cluster has margin at least $g_L^{-1}(4g_U(2\gamma))$, the test passes with high probability. Likewise, if the cluster has margin less than $\gamma$, the test fails with high probability. If the test fails, we skip the iteration of the weak learner. 
\\
\\
How does this modification affect our reliability and coverage? A point can only be mislabeled if the test passes on a cluster with margin less than $\gamma$. Over all iterations of the learner, the probability of this occurring is less than $1-\delta_r$, so our reliability guarantee is maintained up to a constant. To analyze coverage, note that Lemma~\ref{cluster:id} only infers points within the $d$-convex hull of the cluster C. Thus, if C is $\gamma/d$-cluster which does not have margin $g_L^{-1}(4g_U(2\gamma))$, it infers points within at most a $2g_L^{-1}(4g_U(2\gamma))$ margin. Since we set $\gamma$ such that this region has at most $\varepsilon/2$ probability mass, we lose at most $\varepsilon/2$ coverage for skipping clusters with margin less than $g_L^{-1}(4g_U(2\gamma))$. We are left then with the loss in coverage caused by our test failing on a cluster with margin at least $g_L^{-1}(4g_U(2\gamma))$. Since this only occurs with probability $1-\delta_r$ by Lemma~\ref{margin-test}, for small enough $\delta_r$ this only changes the constant on the coverage probability of our weak learner, and thus has no asymptotic affect.
\\
\\
The total query complexity is then given by the complexity for running Theorem~\ref{TNC} $O(\log(1/\delta_u))$ times with the appropriate parameters:
\[
q(\varepsilon,\delta_r,\delta_u) =\tilde{\bigo} \left (\frac{d^{11}}{g_L(g_U^{-1}(\varepsilon_T/2))^{14}}\log^{2}\left( \frac{1}{\delta_r} \right)\log\left(\frac{1}{\delta_u}\right) \log^2\left(\frac{1}{\varepsilon}\right)\right )
\]
\end{proof}
Since s-concave distributions satisfy the requisite distributional properties \cite{s-concave},
\begin{corollary}\label{dd:cor:TNC}
The hypothesis class $(\R^d,H_d)$ is ARPU-learnable under model $(TNC(m,M,\kappa,\varepsilon_0), \mathcal{ISC}_d)$ with query complexity:
\begin{align*}
q(\varepsilon,\delta_r,\delta_u) &=\tilde{\bigo} \left( \frac{2^{14\kappa} M^{42} d^{14 \kappa-3}}{m^{56}\varepsilon'^{14(\kappa-1)}}\log^{2}\left( \frac{1}{\delta_r} \right)\log\left (\frac{1}{\delta_u} \right)\right )\\
\varepsilon' &= \min\left(\frac{\varepsilon}{16},\frac{\varepsilon_0}{2}\right).
\end{align*}
\end{corollary}
When compared with the query complexity of the label only PAC-learning algorithm of \cite{hanneke2015minimax}, Corollary~\ref{dd:cor:TNC} only shows improvement for a small range of parameters $1 < \kappa < \frac{15}{14}$. However, it is not clear to the authors that Hanneke and Yang's algorithm can be extended to an ARPU learner without substantially increasing the query complexity with respect to dimension.

\bibliographystyle{unsrt}  
\bibliography{references} 

\end{document}